\documentclass{article}

\usepackage{arxiv}
\usepackage{natbib}

\usepackage[utf8]{inputenc} 
\usepackage[T1]{fontenc}    
\usepackage{hyperref}       
\usepackage{url}            
\usepackage{booktabs}       
\usepackage{amsfonts}       
\usepackage{nicefrac}       
\usepackage{microtype}      
\usepackage{lipsum}
\usepackage{graphicx}
\graphicspath{ {./images/} }

\usepackage{microtype}
\usepackage{graphicx}
\usepackage{subcaption}
\usepackage{booktabs} 

\usepackage{hyperref}

\newcommand{\Qq}[1][t]{Q^{(#1)}}
\newcommand{\Aa}[1][t]{A^{(#1)}}
\newcommand{\Aaa}[1][t]{A}
\newcommand{\Pl}[1][t]{P^{(#1)}}
\newcommand{\Zz}[1][t]{Z^{(#1)}}

\newcommand{\ddd}{u}
\newcommand{\loss}{L(A,\ddd)}

\newcommand{\E}[2]{\mathbb{E}_{#1}\Big[#2\Big]}

\newcommand{\multiiindex}{\mathbf i}

\newcommand{\TF}[1]{\text{TF}_L({#1}; Q,P)}

\usepackage{amsmath}
\newcommand\numberthis{\addtocounter{equation}{1}\tag{\theequation}}

\newcommand{\tr}[1]{\text{tr}(#1)}
\newcommand{\Aopt}{A^{\text{opt}}}
\newcommand{\dopt}{d^{\text{opt}}}
\newcommand{\LSA}[1]{\text{Attn}^{lin}\left({#1}; W_{k,q,v}\right)}
\newcommand{\LSAa}[1]{\text{Attn}^{lin}\left({#1}; Q,P\right)}
\newcommand{\covariance}{\Sigma^*}

\usepackage{amsmath}
\usepackage{amssymb}
\usepackage{mathtools}
\usepackage{amsthm}

\usepackage[capitalize,noabbrev]{cleveref}

\theoremstyle{plain}
\newtheorem{theorem}{Theorem}[section]
\newtheorem{proposition}[theorem]{Proposition}
\newtheorem{lemma}[theorem]{Lemma}
\newtheorem{corollary}[theorem]{Corollary}
\theoremstyle{definition}

\newtheorem*{remark}{Remark}

\usepackage[textsize=tiny]{todonotes}

\title{Can Looped Transformers Learn to Implement Multi-step Gradient Descent for In-context Learning?}

\author{
 Khashayar Gatmiry \\
  MIT\\
  \texttt{gatmiry@mit.edu} \\
   \And
 Nikunj Saunshi \\
   Google Research\\
  \texttt{nsaunshi@google.com} \\
  \And
 Sashank J. Reddi\\
  Google Research\\
  \texttt{sashank@google.com} \\
  \And
 Stefanie Jegelka \\
  MIT \\
  \texttt{stefje@csail.mit.edu} \\
  \And
 Sanjiv Kumar\\
  Google Research\\
  \texttt{sanjivk@google.com} \\
}

\begin{document}
\maketitle

\begin{abstract}
\looseness-1The remarkable capability of Transformers to do reasoning and few-shot learning, without any fine-tuning, is widely conjectured to stem from their ability to implicitly simulate a multi-step algorithms -- such as gradient descent -- with their weights in a single forward pass.
Recently, there has been progress in understanding this complex phenomenon from an expressivity point of view, by demonstrating that Transformers can express such multi-step algorithms. However, our knowledge about the more fundamental aspect of its learnability, beyond single layer models, is very limited. In particular, {\em can training Transformers enable convergence to algorithmic solutions}?
In this work we 
resolve this for in-context linear regression with linear {\em looped Transformers} -- a multi-layer model with weight sharing that is conjectured to have an inductive bias to learn fix-point iterative algorithms.
More specifically, for this setting we show that the global minimizer of the population training loss implements multi-step preconditioned gradient descent, with a preconditioner that adapts to the data distribution. Furthermore, we show a fast convergence for gradient flow on the regression loss, despite the non-convexity of the landscape, by proving a novel gradient dominance condition.
To our knowledge, this is the first theoretical analysis for multi-layer Transformer in this setting.
We further validate our theoretical findings through synthetic experiments.
\end{abstract}
\section{Introduction}
\label{sec:intro}

\looseness-1
Transformers \citep{vaswani2017attention}  have completely revolutionized the field of machine learning and have led to state-of-the-art models for various natural language and vision tasks.
Large scale Transformer models have demonstrated remarkable capabilities to solve many difficult problems, including those requiring multi-step reasoning through large language models \citep{brown2020language,wei2022chain}. One such particularly appealing property is their few-shot learning ability, where the functionality and predictions of the model adapt to additional context provided in the input, without having to update the model weights. This ability of the model, typically referred to as ``in-context learning'', has been crucial to their success in various applications. Recently, there has been a surge of interest to understand this phenomenon, particularly since \citet{garg2022can} empirically showed that Transformers can be {\em trained} to solve many in-context learning problems based on linear regression and decision trees.
Motivated by this empirical success,  ~\citet{von2023transformers,akyurek2022learning}  theoretically showed the following intriguing expressivity result: multi-layer Transformers with linear self-attention can implement gradient descent for linear regression where each layer of Transformer implements one step of gradient descent. In other words, they hypothesize that the in-context learning ability results from approximating gradient-based few-shot learning within its forward pass. \citet{panigrahi2023trainable}, further, extended this result to more general model classes.

While such an approximation is interesting from the point of view of expressivity, it is unclear if the Transformer model can {\em learn} to implement such algorithms. To this end,~\citet{ahn2023transformers,zhang2023trained} theoretically show, in a Gaussian linear regression setting, that the global minimizers of a one-layer model essentially simulate a single step of preconditioned gradient descent, and that gradient flow converges to this solution.
\citet{ahn2023transformers} further show for the multi-layer case that a single step of gradient descent can be implemented by some stationary points of the loss.
However, a fundamental characterization of all the stationary points for multi-layer Transformer, and the convergence to a stationary point that implements multi-step gradient descent, remains a challenging and important open question.

In this work, we focus our attention on the {\em learnability} of such multi-step algorithms by Transformer models.
Instead of multi-layer models, we consider a closely related but different class of models called {\em looped Transformers}, where the same Transformer block is looped multiple times for a given input.
Since the expectation from multi-layer models is to simulate an iterative procedure like multi-step gradient descent, looped models are a fairly natural choice to implement this.
There is growing interest in looped models with recent results \citep{giannou2023looped} theoretically showing that the iterative nature of the {\em looped Transformer} model can be used to simulate a programmable computer, thus allowing looped models to solve problems requiring arbitrarily long computations.
Looped Transformer models are also conceptually appealing for learning iterative optimization procedures --- the sharing of parameters across different layers, in principle, can provide a better inductive bias than multi-layer Transformers for learning iterative-optimization procedures.
In fact, by employing a regression loss at various levels of looping, \citet{yang2023looped} empirically find that looped Transformer models can be trained to solve in-context learning problems, and that looping on an example for longer and longer at test time converges to a desirable fixed-point solution, thus leading them to conjecture that {\em looped models can learn to express iterative algorithms\footnote{The algorithm should converge to a desirable fixed point.}}.

\looseness-1Despite these strong expressivity results for looped models and their empirically observed inductive bias towards simulating iterative algorithms, very little is known about the optimization landscape of looped models, and the theoretical convergence to desirable and interpretable iterative procedures.
In fact, a priori it is not clear why training should even succeed given that looped models heavily use weight sharing and thus do not enjoy the optimization benefits of overparameterization that has been well studied \citep{buhai2020empirical,allen2019learning}.
In this work, we delve deeper into the problem of optimizing looped Transformers and theoretically study their landscape and convergence for in-context linear regression under the Gaussian data distribution setting used in \citep{ahn2023transformers,zhang2023trained}.
In particular, the main contributions of our paper are as follows:

\begin{itemize}
     \item We obtain a precise characterization 
     of the global minimizer of the population loss for a linear looped Transformer model, and show that it indeed implements multi-step preconditioned gradient descent with pre-conditioner close to the inverse of the population covariance matrix, as intuitvely expected. 
     \item Despite the non-convexity of the loss landscape, we prove the convergence of the gradient flow for in-context linear regression with looped Transformer. To our knowledge, ours is the first such convergence result for a network beyond one-layer in this setting.
     \item To show this convergence, we prove that the loss satisfies a novel gradient-dominance condition, which guides the flow toward the global optimum. We expect this convergence proof to be generalizable to first-order iterative algorithms such as SGD with gradient estimate using a single random instance~\cite{de2022gradient}. 
    \item We further translate having a small sub-optimality gap, achieved by our convergence analysis, to the proximity of the parameters to the global minimizer of the loss. 
\end{itemize}

\section{Related Work}

    \looseness-1\textbf{In-context learning.}
    Language models, especially at larger scale, have been shown to empirically demonstrate the intriguing ability to in-context learn various tasks on test data~\cite{brown2020language}
    More recently, \citet{garg2022can} formalized in-context learning ability and empirically observed that Transformers are capable of in-context learning some hypothesis classes such as linear or two layer neural networks, sometimes improving over conventional solvers. 
    There have since been many paper studying this intriguing in-context learning phenomenon \citep{xie2022an,akyurek2022learning,von2023transformers,bai2023transformers}
    
    \textbf{Transformers in modeling iterative optimization algorithms.} \citet{he2016residual} first observed that neural networks with residual connections are able to implicitly implement gradient descent. \citet{von2023transformers,akyurek2022learning} use this line of reasoning for in-context learning by constructing weights for linear self-attention layers that can emulate gradient descent for various in-context learning tasks, including linear regression. Furthermore, \citet{akyurek2022learning} empirically investigate various in-context learners that Transformers can learn as a function of depth and width. Also, \citet{von2309uncovering} hypothesize the ability of Transformers to (i) build an internal loss based on the specific in-context task, and (ii) optimize over that loss via an iterative procedure implemented by the Transformer weights. 
    \citet{panigrahi2023trainable} generalize the results to show that Transformers can implement gradient descent over a smaller Transformer.
    Recently, \citet{fu2023transformers} empirically observe that Transformers can learn to emulate higher order algorithms such as Newton's method that converge faster than gradient descent.
    
    \looseness-1\textbf{Transformers in reasoning and computation.}
    Indeed the in-context capabilities of Transformers in doing reasoning at test time and emulating an input-specific algorithm as a computer bear deep similarities~\cite{dasgupta2022language,chung2022scaling,lewkowycz2022solving}. Years before the advent of Transformers,~\cite{siegelmann1994computational} study the Turing completeness of recurrent neural networks. To show the computational power of Transformer as a programmable device,~\cite{perez2019turing,perez2021attention,wei2022statistically} demonstrate that Transformers can simulate Turing machines. Furthermore, \citet{mcgrathtracr} propose using Transformer as programmable units and construct a compiler for the domain specific programming language called RASP.
    \citet{perez2019turing} further find a more efficient implementation of a programming language that is also Turing complete using looped Transformers, without scaling with the number of lines of code.
    More recently \citet{giannou2023looped} used looped models to simulate a single-instruction program.
    \citet{yang2023looped} show that looped Transformers can in-context learn solvers for linear regression or decision trees as well as normal Transformers but with much fewer parameters.

\section{Preliminaries}
\subsection{In-context learning (ICL)}
One of the surprising emergent abilities of large language models is their ability to adapt to specific learning tasks without requiring any additional fine tuning. 
    Here we restate the formalism of in-context learning introduced by~\citet{garg2022can}. Suppose for a class of functions $\mathcal F$ and input domain $\mathcal X$, we sample an in-context learning instance $\mathcal I = (\{x_i, y_i\}_{i=1}^n, x_q)$ by sampling $x_i\sim \mathcal D_{\mathcal X}$ and $f \sim \mathcal D_{\mathcal F}$ independently, then calculating $\forall i\in \{1,2,\dots,n\}, y_i = f(x_i)$. An in-context learner $M_\theta$ parameterized by $\theta$ is then a mapping from the instance $\mathcal I$ to a prediction for the label of the query point $f(x_q)$. The population loss of $M_\theta$ is then defined as
    \begin{align}
        \loss(M_\theta) = \E{\mathcal D_f, \mathcal D_{\mathcal X}}{\Big(M_\theta(\mathcal I) - f(x_q)\Big)^2}\label{eq:populationloss}
    \end{align}

\subsection{Linear regression ICL setup}\label{subsec:setup}
In this work, we consider linear regression in-context learning; namely, we assume sampling a linear regression instance is given by $\mathcal I = \Big(\Big\{x_i, y_i\Big\}_{i=1}^n, x_q\Big)$ where
for $w^* \sim \mathcal N(0,{\Sigma^*_{d\times d}}^{-1})$, $ \ x_i \sim \mathcal N(0,\Sigma^*_{d\times d})$ we have $y_i = f_{w^*}(x_i) = {w^*}^\top x_i$ for all $i \in [n]$.
The goal is to predict the label of $x_q$, i.e. ${w^*}^\top x_q$. 
Define the data matrix $X \in \mathbb R^{d\times n}$, whose columns are the data points $\{x_i\}_{i=1}^n$:

\begin{align*}
    X = [x_1, \dots, x_n]
\end{align*}

\looseness-1We further assume $n>d$, i.e. the number of samples is larger than the dimension.
This combined with the fact that $\mathcal I$ is realizable implies that we can recover $w^*$ from $\Big\{x_i, y_i\Big\}_{i=1}^n$ by the well-known pseudo-inverse formula:
\begin{align*}
    w^* = (XX^\top)^{-1}Xy.
\end{align*}

 While a reasonable option for the in-context learner $M_\theta(\mathcal I)$ is to implement $(XX^\top)^{-1}Xy$, matrix inversion is arguably an operation that can be costly for Transformers to implement. 
 On the other hand, it is known that linear regression can also be solved by first order algorithms that move along the negative gradient direction of the loss
 \begin{align*}
     \ell_2^2(w) = \|X^\top w^* - y\|^2.
 \end{align*}
\looseness-1Using a standard analysis for smooth convex optimization, since the Hessian of the loss $\|X^\top w^* - y\|^2$ is $XX^\top$ with condition number $\kappa$, gradient descent with step size $\frac{1}{\kappa}$ converges in $O(\kappa)$ iterations. This means that we need $O(\kappa)$ many layers in the Transformer to solve linear regression. Particularly, \citet{von2023transformers} show a simple weighting strategy for the key, query, and value matrices of a linear self-attention model so that it implements gradient descent, which we introduce in the next section.

\subsection{Linear self-attention layer}
Here we define a single attention layer that forms the basis of the linear Transformer model we consider.
Define the matrix $Z^{(0)}$, which we use as the input prompt to the Transformer, by combining the data matrix $X$, their labels $y$, and the query vector $x_q$ as
\begin{align*}
    Z^{(0)} = 
    \begin{bmatrix}
            X & x_q \\
             y^\top & 0
        \end{bmatrix}.
\end{align*}

Following~\cite{ahn2023transformers,schlaglinear,von2023transformers}, we consider the linear self-attention model $\LSA{Z}$ defined as
\begin{align*}
    &\LSA{Z} \coloneqq W_v Z M (Z^\top W_k^\top W_q Z),\\
    &M \coloneqq \begin{bmatrix}
            I_{n\times n} & 0 \\
             0 & 0
        \end{bmatrix} \in \mathbb R^{(n+1) \times (n+1)},
\end{align*}
where $W_k,W_q,W_v$ are the key, query, and value matrices, respectively and the index $k\times r$ below a matrix determines its dimensions. Furthermore, similar to ~\cite{ahn2023transformers}, we use mask matrix $M$ in order to avoid the tokens corresponding to $(x_i,y_i)$ to attend the query vector $x_q$, and combine product of the key and query matrices into $Q = W_k^\top W_q$ to obtain the following parameterization for the attention layer (denoting $W_v$ by $P$):
\begin{align}
    &\LSAa{Z} \coloneqq P Z M (Z^\top Q Z).
    \label{eq:linear_attn}
\end{align}

\subsection{Linear looped Transformer}
The linear looped transformer $\TF{\Zz[0]}$ can be defined by 
simply chaining $L$ linear self-attention layers with shared parameters $Q$ and $P$. In particular, we define
\begin{align}
    \Zz[t] \coloneqq \Zz[t-1] - \frac{1}{n}\LSAa{Z}.\label{eq:tfupdate}
\end{align}
for all $t \in [L]$. Then, the output of an $L$ layer looped transformer $\TF{\Zz[0]}$ just uses the $(d+1)\times (n+1)$ entry of matrix $\Zz[L]$ i.e.,
\begin{align}
        \TF{\Zz[0]}= -{\Zz[L]}_{(d+1),(n+1)}.\label{eq:tfone}
\end{align}
We note that the minus sign in the final output of the Transformer is only for simplicity of our expositions later on.

{\bf Can looped Transformer implement multi-step gradient descent?} We first examine the expressivity of looped Transformer. The key idea is to leverage the existing result of one step preconditioned gradient descent from \cite{ahn2023transformers} and use the loop structure of looped Transformer to show that it can implement multi-step preconditioned gradient descent.
For completeness, we first restate the observation of~\cite{ahn2023transformers} that linear attention can implement a step of preconditioned gradient descent with arbitrary preconditioner $A$. For this, it is enough to pick

\begin{proposition}[Expressivity; Lemma 1 from~\cite{ahn2023transformers}]
\label{prop:precond_gd}
    For an appropriate choices of $P$ and $Q$, the linear looped Transformer from \Cref{eq:tfone} implements multiple steps of preconditioned gradient descent on the linear regression instance.
\end{proposition}
The proof described below is identical to the one from \citep{ahn2023transformers}. While this is an expressivity result, our main contribution in later sections is to show convergence to such a solution.

\begin{proof}
For \Cref{prop:precond_gd}, preconditioned gradient descent with preconditioner $A$ can be implemented by setting $P$ and $Q$ to the following:
\begin{align}
    Q \coloneqq \begin{bmatrix}
            A_{d\times d} & 0 \\
             0 & 0
        \end{bmatrix},
    P \coloneqq \begin{bmatrix}
            0_{d\times d} & 0\\
            0 & 1
        \end{bmatrix}.\label{eq:qp}    
\end{align}
Then for the matrix $\begin{bmatrix}
            X & x_q 
        \end{bmatrix}\in \mathbb R^{d\times (n+1)}$
\begin{align}
  {\Big[\LSAa{Z}\Big]}_{(d+1),} &= y^\top X^\top A\begin{bmatrix}
            X & x_q 
        \end{bmatrix} \\
   = -(0-&\frac{1}{n}A\nabla_w \ell_2^2(0))^\top X,\numberthis\label{eq:initialupdateform}\\
    {\Big[\LSAa{Z}\Big]}_{1:d,} &= 0_{d\times n},\label{eq:initialupdateformsec}
\end{align}
 where index $(k:r,)$ denotes the restriction of the matrix to its rows between $k$ and $r$, and we used the fact that $\nabla_w \ell_2^2(0) = Xy$. But if we update $w$ with the gradient of $\ell_2^2(w)$ preconditioned by $A$ and step size $\frac{1}{n}$ and assuming $w_0 = 0$, then
 \begin{align*}
     w_1 = w_0 - \frac{1}{n}A\nabla_w \ell_2^2(w_0) = 0 - \frac{1}{n}A\nabla_w \ell_2^2(0).
 \end{align*}
 Plugging this into Equation~\eqref{eq:initialupdateform}:
 \begin{align*}
     &{\Big[\LSAa{Z}\Big]}_{(d+1),1:n} = -w_1^\top X\\
     &{\Big[\LSAa{Z}\Big]}_{(d+1),n+1} = -w_1^\top x_q.
 \end{align*}
 Further, by using Equation~\eqref{eq:initialupdateform} in Equation~\eqref{eq:tfupdate} we get
 \begin{align*}
     &\Big[\Zz[1]\Big]_{(d+1),1:n} = y^\top - w_1^\top X, \\
     &\Big[\Zz[1]\Big]_{(d+1),n+1} =  - w_1^\top x_q,\Big[\Zz[1]\Big]_{1:d,} = X.\numberthis\label{eq:finalupdate}
 \end{align*}
 It is easy to see that Equations~\eqref{eq:finalupdate} hold for all $\Zz[t]$ with $w_1$ substituted by corresponding $w_t$, thus, allowing implementation of multi-step gradient descent.
 \end{proof}

\subsection{Loss function on the weights}\label{sec:lossonweights}
\looseness-1In previous section, while we observed that looped Transformer can implement preconditioned gradient descent, the choice of the preconditioner and its learnability by optimizing a loss function (e.g. squared error loss) still remain unclear.  Following~\cite{ahn2023transformers,zhang2023trained}, we search for the best setting of matrices $P,Q$, where  $Q \coloneqq \begin{bmatrix}
            A_{d\times d} & 0 \\
             0 & 0
        \end{bmatrix}$, i.e. only the top left $d\times d$ block can be non-zero, and
        $P \coloneqq \begin{bmatrix}
            0_{d\times d} & 0\\
            u^\top & 1
        \end{bmatrix}$
        for parameter vector $u\in \mathbb R^d$.


The population squared loss as a function of $A$ and $\ddd$ is
\begin{align*}
    &\loss = \E{w^*, X}{(\TF{Z_0} - y_q)^2}.
\end{align*}
\looseness-1We define a parameter $\delta \coloneqq \Big(\frac{8Ld}{\sqrt n}\Big)^{1/(2L)}$ which governs the accuracy of our estimates, which goes to zero as $n\rightarrow \infty$.



\begin{table*}[t]
\centering
\caption{Summary of main theoretical results and key assumptions}
\label{table:summary}
\begin{tabular}{l| l| l}
Results  & Initialization & Basic Description\\
\hline
\hline
Theorem~\ref{thm:globalsol}
& Arbitrary & \begin{tabular}{@{}l@{}} For the global minimizer of the loss $(A^{opt},{\ddd}^{opt})$, ${\ddd}^{opt} = 0$ \\ and the preconditioner part $A^{opt}$ is close to ${\covariance}^{-1}$.\end{tabular}\\
\hline
Theorem~\ref{thm:convergenceflow} & $\ddd = 0, \covariance = I$ &\begin{tabular}{@{}l@{}}The gradient flow of the loss converge to a loss value with small suboptimality gap,\\
in the proximity of the global minimizer in the parameter space\end{tabular}\\
\hline
Theorem~\ref{thm:gradientdominance} & $\ddd = 0, \covariance = I$ &
\begin{tabular}{@{}l@{}}The loss satisfies a gradient dominance with power $\frac{2L-1}{L}$, \\
given that the suboptimality gap is not too small\end{tabular}
\\
\hline
Theorem~\ref{thm:proximity}
& $\ddd = 0$ & \begin{tabular}{@{}l@{}}Small suboptimality gap implies closeness in the parameter space\\
(In spectral distance). \end{tabular}
\\
\hline
Theorem~\ref{thm:outofdist}
& $\ddd = 0$ & \begin{tabular}{@{}l@{}} Instance-dependent out of distribution generalization for the minimizer of\\
the population loss. \end{tabular}
\end{tabular}
\end{table*}

\subsection{Choice of the preconditioner}
\label{sec:preconditioner}

It is instructive to discuss the choice of the preconditioner $A$ since it determines speed of convergence of $w_i$ to the solution of the regression. Note that the exact solution of an over-determined linear regression instance $(X,y)$ is $w = (X X^\top)^{-1} Xy$. This can be obtained only after one step of preconditioned gradient descent starting from the origin and using inverse of the data covariance matrix preconditioner 
\begin{align*}
    \Sigma = \frac{1}{n}\sum_{i=1}^n x_i x_i^\top.
\end{align*}

In general, it may not be possible to pick the weights of the Transformer to ensure such a preconditioner for all possible regression instances as every instance $(X,y)$ has its own data covariance matrix $\frac{1}{n}XX^\top$. But since $x_i$'s are sampled i.i.d from  $\mathcal N(0,\covariance)$, it is known that the inverse of the data covariance matrix concentrates around the inverse of the population covariance $\covariance$. Thus, a reasonable choice of $A$ is the inverse of the population covariance matrix $ {\covariance}^{-1}$.

\looseness-1In fact,~\citet{ahn2023transformers} show that the global minimum of {\em one-layer} linear self-attention model under Gaussian data is the inverse of the population covariance matrix plus some small regularization term. However, the characterization of global minimizer(s) of the population loss for the multilayer case is largely missing. Specifically, \emph{is there a global optimum for solving regression with Transformers that is close to gradient descent with preconditioner ${\covariance}^{-1}$?}
In this work, we solve this open problem for looped Transformers; given that data is sampled iid from $\mathcal N(0,\covariance)$, we show that the optimal looped Transformer under constraints stated in Section~\ref{sec:lossonweights}, $A^{opt}$ will be close to ${\covariance}^{-1}$.

\subsection{Convergence}
 In this section, we state our main results. First, we give a tight estimate on the set of global minimizers of the population loss, under the Gaussian assumption, for the looped Transformer model with arbitrary number of loops $L$.   
\begin{theorem}[Characterization of the optimal solution]\label{thm:globalsol}
     Suppose $\{\Aopt, \ddd^{opt}\}$ are a global minimizer for $L(A,\ddd)$. Then, under condition 
     $\frac{8Ld^2}{\sqrt n} \leq \frac{1}{2^{2L}}$,
     \begin{enumerate}
         \item 
         $L(\Aopt, \ddd^{opt}) \leq \frac{8Ld^2 2^{2L}}{\sqrt n}$
         \item $(1 - c){\covariance}^{-1} \preccurlyeq \Aopt \preccurlyeq (1+c){\covariance}^{-1},  c=8\delta d^{1/(2L)}$ and ${\ddd}^\text{opt} = 0$,
            where recall $\delta \coloneqq \Big(\frac{8Ld}{\sqrt n}\Big)^{1/(2L)}$.
     \end{enumerate}
\end{theorem}

\begin{remark}
From Theorem~\ref{thm:globalsol}, we first observe that the parameter $\ddd$ has no effect in obtaining a better regression solver and has to be set to zero in the global minimizer. This result was not known in the previous work~\cite{ahn2023transformers}. A value of ${\ddd}^\text{opt}=0$ implies that the optimal looped Transformer exactly implements $L$ steps of preconditioned gradient descent, with preconditioner $\Aopt$.
\end{remark} 

\looseness-1Secondly, as discussed in \Cref{sec:preconditioner}, the choice of preconditioner plays an important role in how fast gradient descent converges to the solution of linear regression.
Intuitively the inverse of the population covariance seems like a reasonable choice for a single fixed preconditioner, since it is close to the inverse of the data covariance for all linear regression instances.
The above result shows that the global optimum is indeed very close to the inverse of the population covariance.

Precisely how close the optimum is to the population covariance depends on the parameter $\delta = \frac{4kd}{\sqrt n}$, which goes to zero as the number of examples in each prompt goes to infinity.
In general, we do not expect the global minimizer to be exactly equal to ${\Sigma^*}^{-1}$.
Indeed for the case of one layer Transformer, which is equivalent to a loop-transformer with looping parameter $L = 1$, the global minimizer found in \citet{ahn2023transformers} is not exactly the inverse of the covariance matrix, but close to it.
Even in their case, the distance goes to zero as $n\rightarrow \infty$. This shows that our estimate in Theorem~\ref{thm:globalsol} is essentially the best that one can hope for.
Next, we state our second result, which concerns the convergence of the gradient flow of the loss to the proximity of the global minimizer.

\begin{theorem}[Convergence of the gradient flow]\label{thm:convergenceflow}
    Consider the gradient flow with respect to the loss $L(A,0)$ for $\Sigma^* = I$:
    \begin{align*}
        \frac{d}{dt}A(t) = -\nabla_A L(A(t),0). 
    \end{align*}
    Then, for any $\xi \geq 2(4\delta)^{2L}$, after time $t \geq \Big(\frac{1}{\xi}\Big)^{(L-1)/L} (\frac{16 L}{L-1})^{(L-1)/(2L-1)}$ we have 
    \begin{enumerate}
        \item $L(A(t)) \leq \xi$,
        \item $(1-8(1 + 4d^{1/(2L)})\xi^{1/(2L)}){A^{opt}}\preccurlyeq A$\\
        $~~~\preccurlyeq (1 + 8(1 + 4d^{1/(2L)})\xi^{1/(2L)} )A^{opt}$
    \end{enumerate}
\end{theorem}

Note that the landscape of the loss with respect to $A$ is highly non-convex, hence Theorem~\ref{thm:convergenceflow} does not follow from the typical convex analysis of gradient flows. The key in obtaining this result is that we show a novel gradient dominance condition for the loss with power $(2L-1)/L$, which we state next.

\begin{theorem}[Gradient dominance]\label{thm:gradientdominance}
    Given $\Sigma^* = I$, for any $A$ that $L(A,0) \geq \frac{16Ld 4^L}{\sqrt n}$
    , we have the following gradient dominance condition:
    \begin{align*}
        \|\nabla_A L(A,0)\|^2 \geq \frac{1}{16}L(A,0)^{(2L-1)/L}.
    \end{align*}
\end{theorem}

\begin{remark}
Theorem~\ref{thm:gradientdominance} illustrates that 
the squared norm of the gradient is at least proportional to the power $(2L-1)/L$ of the value of the loss. On the other hand, it is easy to see that the speed of change of the value of the loss on the gradient flow, namely $\frac{d}{dt}L(A(t),0)$, is equal to the squared norm of the gradient.  But when the value of the loss is large, then the size of the gradients increase accordingly due to gradient dominance, therefore the convergence is faster when the loss is high. This trend is evident in the rigorous rate that we obtain on the convergence of the gradient flow in Theorem~\ref{thm:globalsol}.  
\end{remark}

\section{Proof Ideas}
The proof is structured as follows:
\begin{itemize}
    \item \looseness-1We obtain closed form formula for the loss function in Lemma~\ref{lem:matrixformat} in terms of the parameter $A$ and covariance $\covariance$. The loss depends on how close $A^{1/2}\Sigma A^{1/2}$ is close to identity for a randomly sampled $\Sigma$.
    Using the estimates in Lemma~\ref{lem:eigapprox}, we obtain an estimate on the loss based on the eigenvalues of the matrix $A^{1/2}\covariance A^{1/2}$. Importantly, the result of Lemma~\ref{lem:eigapprox} is based on estimating the higher moments of the Wishart matrix with arbitrary covariance, shown in Lemma~\ref{lem:momentcontrol}
    Using our estimate of the loss in Lemma~\ref{lem:eigapprox}, we obtain a precise characterization of the global optimum.
    \item \looseness-1We further use Lemma~\ref{lem:eigapprox} to drive an estimate on the magnitude of the gradient based on the same eigenvalues, those of $A^{1/2}\covariance A^{1/2}$. Comparing this with our estimate for the loss from Lemma~\ref{lem:eigapprox}, we obtain the gradient dominance condition in Theorem~\ref{thm:gradientdominance}.
    \item \looseness-1We use the gradient dominance condition to estimate the speed of convergence of the gradient flow to the proximity of the global minimizer in Theorem~\ref{thm:convergenceflow}.
\end{itemize}

\begin{figure*}[t]
    \centering
    \begin{subfigure}{0.45\textwidth}
    \centering    
    \includegraphics[width=\textwidth]{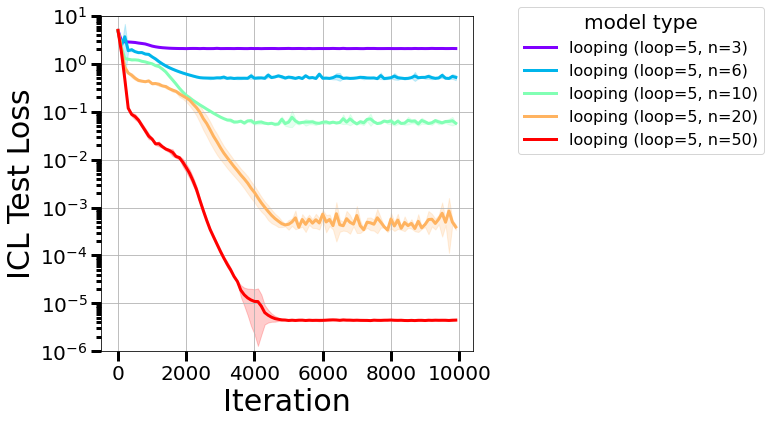}
    \caption{Loss curve}
    \label{fig:looping_loss_vs_N}
    \end{subfigure}
    \begin{subfigure}{0.45\textwidth}
    \centering
    \includegraphics[width=\textwidth]{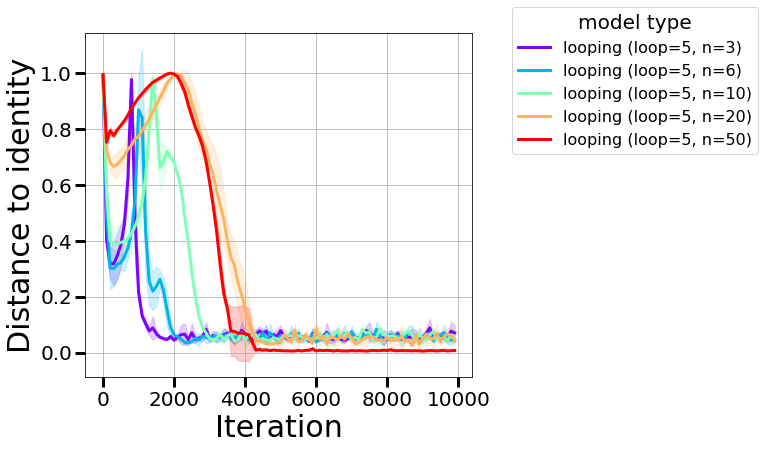}
    \caption{Distance to identity}
    \label{fig:distancetoId_vs_N}
    \end{subfigure}\hfill
    \caption{\looseness-1 Measuring the effect of number of samples $n$ for looped models trained on inputs of dimension $d=5$. \Cref{thm:globalsol} shows that the global optima for looped models is $A = I$ (since $\covariance = I$ here) for large enough $n$. Here we verify that $A$ converges to something very close to $I$ even for smaller values of $n$ and even $n<d$.}
    \label{fig:loss_distance_N}
\end{figure*}

\begin{figure}[tbp]
    \centering
    \includegraphics[width=0.42\textwidth]{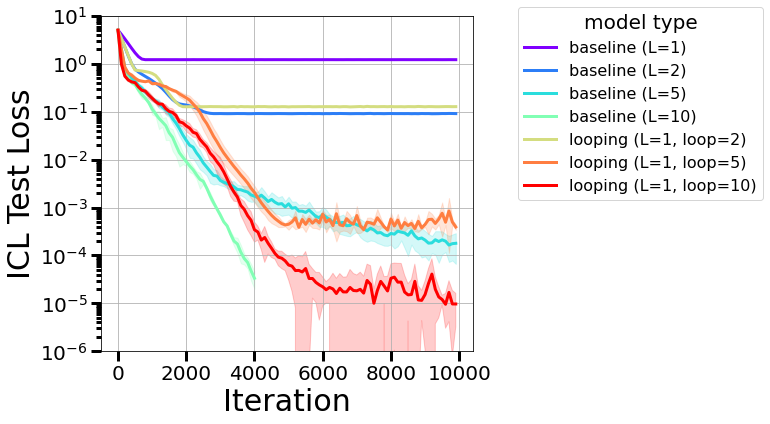}
    \caption{\looseness-1 Loss trajectory as training proceeds for looped models and baseline multilayer models. Interestingly a 1-layer model looped $L$ times performs similarly to an $L$ layer multilayer model. Furthermore, increasing the number of loops leads to lower loss.}
    \label{fig:looping_baseline}
\end{figure}

The starting point of the proof is that we can write the loss in a matrix power format based on $A$ when $\ddd$ is set to zero:
\begin{lemma}\label{lem:matrixformat}
    Given $\ddd = 0$, the loss for looped Transformer is as follows:
    \begin{align*}
        &L(\Aaa, 0) = \E{X}{\tr{(I - A^{1/2}\Sigma A^{1/2})^{2L}}},
    \end{align*}
    where $\Sigma = \frac{1}{n}\sum_{i=1}^n x_i x_i^\top$.
\end{lemma}
\looseness-1To be able to estimate the global minimizers of this loss, first we need to estimate its value. In particular, we hope to relate the value of the loss to the eigenvalues of $A$. Note that if the data covariance matrix $\Sigma$ was equal to the population covariance matrix $\covariance$, then loss would turn into $\tr{(I - A^{1/2}\covariance A^{1/2})^{2L}}$, whose global minimum is $A = {\covariance}^{-1}$. However, we can still hope to approximate the value of $\E{X}{\tr{(I - A^{1/2}\Sigma A^{1/2})^{2L}}}$ with $\tr{(I - A^{1/2}\covariance A^{1/2})^{2L}}$ given that we have a control on the expectation of the powers of the form $\E{X}{(A^{1/2}\Sigma A^{1/2})^k}$ for $1 \leq k \leq 2L$. While there are some work on obtaining formulas for the moments of the Wishart matrix (note that $\Sigma A$ 
 is a Wishart matrix), these formulas~\cite{bishop2018introduction} are in the form of large summations and do not directly provide closed-form estimates in the general case. In general, the moment of the product of $n$ Gaussian scalar variables can be written as a sum over various allocations of the variables into pairs, then multiplying the covariances of the pairs, due to Isserlis' Theorem. However, this gives a formula in terms of a large summation. Here, we propose a simple combinatorial argument in Lemma~\ref{lem:momentcontrol} which relates the moments of the Wishart matrix to the cycle structure of certain graphs related to the pairings of the Gaussian vectors, while using Isserlis' theorem. In particular, we show the following Lemma which relates the eigenvalues of the moments of the data covariance matrix and the covariance matrix itself:
\begin{lemma}[Moment controls]\label{lem:momentcontrol}
    Suppose $\forall i\in [n], \tilde x_i\sim \mathcal N(0,\tilde \Sigma)$. Consider the eigen-decomposition  $\tilde \Sigma = \sum_{i=1}^d \lambda_i u_i u_i^\top$ with eigenvalues $\lambda_1 \geq \lambda_2\geq \dots\geq \lambda_d$. Then, for all $1\leq k\leq 2L$, $\E{}{(\frac{1}{n}\sum_{i=1}^n \tilde x_i \tilde x_i^\top)^{k}}$ can be written as 
    \begin{align*}
        \E{}{(\frac{1}{n}\sum_{i=1}^n \tilde x_i \tilde x_i^\top)^{k}} = \sum_{j=1}^d \alpha^{(j)}_{n,d,k} u_j u_j^\top,
    \end{align*}
    where for all $1\leq j\leq d$:
    \begin{align*}
        \lambda_j^k-\delta^k\lambda_1^k \leq \alpha^{(j)}_{n,d,k} \leq \lambda_j^k + \delta^k\lambda_1^k. 
    \end{align*}
\end{lemma}
Next, we translate this Lemma to a control over the eigenvalues of $A^{1/2}\E{}{(I - \Sigma A)^{k}}A^{-1/2}$ with respect to that of $A^{1/2} \covariance A^{1/2}$:

\begin{lemma}[Eigenvalue approximation]\label{lem:eigapprox}
Given the eigen-decomposition
\begin{align*}
    A^{1/2}\covariance A^{1/2} = \sum_{i=1}^n \lambda_i u_i u_i^\top,
\end{align*}
then for all $k\leq 2L$, the matrix $\E{}{(I - A^{1/2}\Sigma A^{1/2})^{k}}$ can be written as 
\begin{align*}
    \E{}{(I - A^{1/2}\Sigma A^{1/2})^{k}}= \sum_{i=1}^n \beta^{(k)}_i u_i u_i^\top, 
\end{align*}
where
\begin{align*}
    (1 - \lambda_i)^k - \delta^k (\lambda_1 + 1)^k &\leq \beta^{(k)}_i \\
    &\leq (1 - \lambda_i)^k + \delta^k (\lambda_1 + 1)^k.
\end{align*}
\end{lemma}

We use Lemma~\ref{lem:eigapprox} to argue that the matrix $(I - A^{1/2}\Sigma A^{1/2})^{2L}$ for data covariance matrix $\Sigma = \frac{1}{n}\sum_{i=1}^n x_i x_i^\top$ roughly behaves like $(I - A^{1/2}\Sigma^* A^{1/2})^{2L}$, plus some noise on each eigenvalue.

\textit{The rest of the proof in a high level goes as follows:} 
Given that the value of the loss has certain amount of sub-optimality gap, we deduce, using Lemma~\ref{lem:eigapprox}, a lower bound on the distance of the eigenvalues of $A^{1/2}\Sigma A^{1/2}$ from one in $2L$-norm. We then again apply Lemma~\ref{lem:eigapprox}, this time for power $2L-1$, which is relevant from the algebraic form of the gradient $\nabla_A L(A,0)$, to deduce a lower bound for norm of the gradient based on the distance of the eigenvalues of $A^{1/2}\Sigma A^{1/2}$ to one in the $4L-2$-norm. Finally by relating these two results using Holder inequality, we obtain the gradient dominance for values of sub-optimality that are not too small. 

Note that as in Lemma~\ref{lem:eigapprox}, the magnitude of the noise on all the eigenvalues is controlled by the largest eigenvalue, hence the noise is multiplicative only for the largest eigenvalue. This introduces additional difficulty in arguing about the distance of eigenvalues of $A$ from one given a certain suboptimality gap.
Next, using the gradient dominance condition, we estimate the gradient flow ODE and upper bound the value of the loss at a positive time $t > 0$ in~\Cref{thm:convergenceflow}. 
To finish the proof of Theorem~\ref{thm:convergenceflow}, we need to translate a small suboptimality gap into closeness to global optimum, which we prove the following Theorem:
\begin{theorem}[Small loss implies close to optimal]\label{thm:proximity}
    For $\epsilon > 4\delta$, if $L(A,0) \leq \epsilon^{2L}/2$, then for $c=(4 + 16d^{1/(2L)})$
    \begin{align*}
         (1-c\epsilon )&{A^{opt}}\preccurlyeq A \preccurlyeq (1+c\epsilon )A^{opt}.
    \end{align*}
\end{theorem}

\subsection{Out of distribution generalization}

In the result below, we show that a looped Transformer learned on one in-context distribution can generalize to other problem instances with different covariance, owing to the fact that it has learned a good iterative algorithm.

\begin{theorem}\label{thm:outofdist}
    Let $A^{opt},\ddd^{opt}$ be the global minimizers of the poplulation loss for looped Transformer with depth $L$ when the in-context input $\{x_i\}_{i=1}^n$ are sampled from $\mathcal N(0,\covariance)$ and $w^*$ is sampled from $\mathcal N(0,{\Sigma^*}^{-1})$. Suppose we are given an arbitrary linear regression instance $\mathcal I^{out} = \Big\{x^{out}_i,y^{out}_i\Big\}_{i=1}^n, w^{out, *}$ with input matrix $X^{out} = [x^{out}_1, \dots, x^{out}_n]$, query vector $x^{out}_q$, and label $y^{out}_q = {w^{out, *}}^\top x^{out}_q$. Then, if for parameter $0 < \zeta < 1$, the input covariance matrix $\Sigma^{out} = X^{out}{X^{out}}^\top$ of the out of distribution instance satisfies
    \begin{align}
       \zeta \covariance \preccurlyeq \Sigma^{out} \preccurlyeq (2-\zeta)\covariance,\label{eq:comparetocovariance}
    \end{align}
    we have the following instance-dependent bound on the out of distribution loss: 
    \begin{align*}
        &(\TF{Z^{out}_0} - y^{out}_q)^2 \\
        & \leq (1+16\delta d^{1/(2L)})^2(1 + 16\delta d^{1/(2L)} - \zeta)^{2L}\\
        &\times \Big\|x^{out}_q\Big\|_{\covariance}^2 \Big\|w^{out, *}\Big\|_{{\covariance}^{-1}}^2.
    \end{align*}    
\end{theorem}





\begin{figure*}[t]
    \centering
    \includegraphics[width=0.72\textwidth]{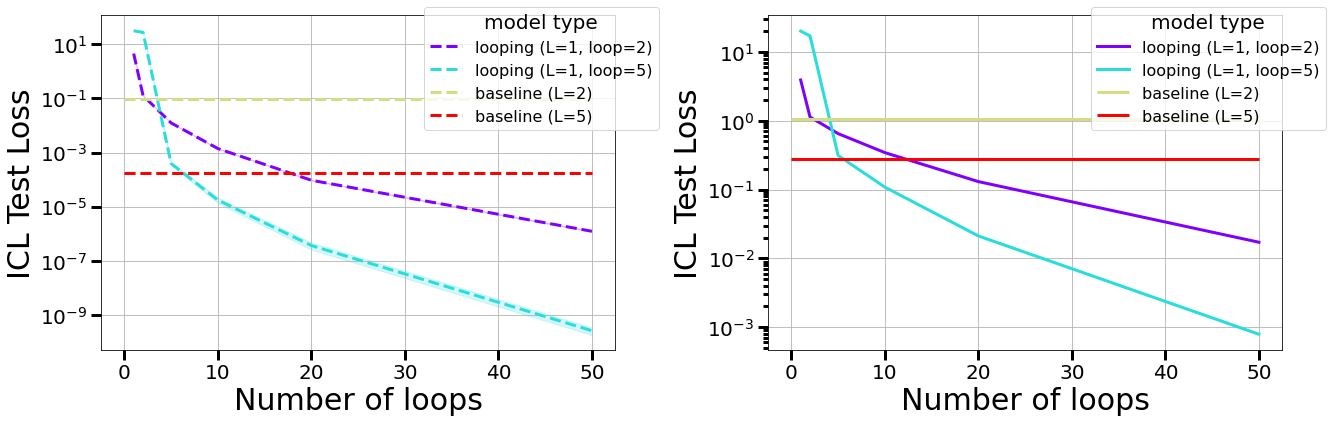}
    \caption{\looseness-1 In-context linear regression loss on in-distribution (left) and out-of-distribution (right) data which is sampled using a different covariance $\Sigma\neq I$. For looped models trained with just few loops (2, or 5), evaluating with more loops keeps improving the loss in both cases, suggesting that it learned the correct iterative algorithm.\vspace{-0.15in}}
    \label{fig:loop_ID_OOD}
\end{figure*}

\begin{figure}[!t]
    \centering
    \includegraphics[width=0.42\textwidth]{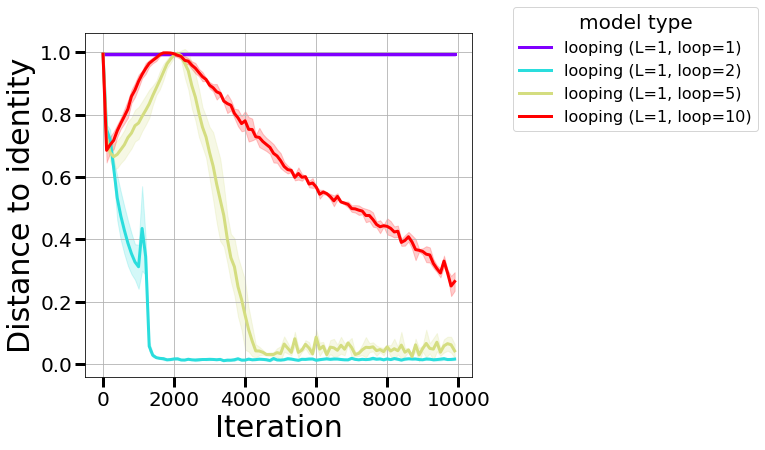}
    \caption{\looseness-1 The iterate $A$ converges to identity for all number of loops. The converge is slower for large number of loops which is also observed by our rate of convergence in~\Cref{thm:convergenceflow}. Interestingly just training with 1 loop does not converge in this setup.\vspace{-0.1in}}
    \label{fig:distance_loop}
    \vspace{-0.1in}
\end{figure}

\section{Experiments}
\label{sec:experiments}

\looseness-1In this section we run experiments on in-context learning linear regression to validate the theoretical results and to go beyond them.
In particular, we test if looped models can indeed be trained to convergence, as the theory suggests, and whether the learned solution is close to the predicted global minima.
Furthermore, we investigate the effect of various factors such number of loops, number of in-context samples and depth of the model (in the multi-layer case).
We use the codebase and experimental setup from \cite{ahn2023transformers} for all our linear regression experiments.
In particular we work with $d=10$ dimensional inputs and train with $L$ attention layer models for multilayer training and 1 layer attention model looped $\ell$ times.
Inputs and labels are sampled exactly based on the setup from \Cref{subsec:setup}, using covariance $\covariance=I$.

\subsection{Effect of loops}

We first test whether training with looped model converges to a low loss, and how small the loss can be made with more loops.
In \Cref{fig:looping_baseline}, we see that looped models indeed converge to very small loss very quickly, and higher loops leads to lower loss as expected.
Interestingly, we find that a 1-layer model looped $L$ times roughly has very loss to an $L$-layer non-looped model.

\subsection{Effect of in-context samples}

\looseness-1\Cref{thm:convergenceflow} shows convergence of the gradient flow for looped models when the number of in-context samples, $n$, is large compared to the dimension $d$.
In these experiments we test the convergence of loss and iterate for smaller values of $n$, when $n$ is closer to, or even smaller than $d$.
In \Cref{fig:loss_distance_N} we observe that the loss converges for all values of $n>1$ and the iterates also converge to a value very close identity.
Theoretically proving this result remains an open question.

\subsection{Out-of-distribution evaluation}

\looseness-1While the looped model was trained with linear regression instances with identity covariance, we evaluate the trained looped model on out-of-distribution (OOD) data with a different covariance $\Sigma\neq I$.
\Cref{thm:outofdist} predicts that the model trained on identity covariance should also generalize to other covariances, because it simulates multi-step preconditioned gradient descent that works for all problems instances.
In \Cref{fig:loop_ID_OOD}, we find that the learned looped model achieves small loss for OOD data, although the scale of the loss is higher than in-distribution (ID) data.
Interestingly, for looped models trained with just 2 (or 5) loops, evaluating them with arbitrarily large number of loops during test time continues to decrease the loss even further for ID and OOD data.
This suggests that the trained looped models are indeed learning a good iterative algorithm.

\section{Conclusion}

\looseness-1This work provides the first convergence result showing that attention based models can {\em learn} to simulate multi-step gradient descent for in-context learning.
The result not only demonstrates that Transformers can learn interpretable multi-step iterative algorithms (gradient descent in this case), but also highlights the importance of looped models in understanding such phenomena.
There are several open questions in this space including understanding the landscape of the loss, convergence of training without weight sharing across layers, and handling of non-linearity in the attention layers.
It is also interesting to understand the empirical phenomenon that looping the trained models beyond the number of loops used in training can continue to improve the test loss. 
One way to show this is by obtaining a tighter upper bound on the optimal loss value.

\newpage
\bibliography{main}
\bibliographystyle{plainnat}

\newpage
\appendix
\onecolumn
\section{Gradient Dominance and Convergence of SGD in multilayer Transformers}

\subsection{A formula for the loss in the multilayer case}

\begin{theorem}\label{thm:lineartfformula}
     Consider the linear attention layer with matrices $P,Q$ set as in Equation~\eqref{eq:qp} but with different parameters for different layers (i.e. without weight sharing). Namely, suppose for the layer $t$ attention, we set 
     \begin{align}
    \Qq[t] \coloneqq \begin{bmatrix}
            \Aa[t]_{d\times d} & 0 \\
             0 & 0
        \end{bmatrix},
    \Pl[t] \coloneqq \begin{bmatrix}
            0_{d\times d} & 0\\
            {\ddd^{(t)}}^{\top} & 1
        \end{bmatrix}.\label{eq:qp}    
\end{align}
     Now defining
     \begin{align*}
     &\Big[\Zz[t]\Big]_{(d+1),1:n} = y^{(t)}, \\
     &\Big[\Zz[t]\Big]_{(d+1),(n+1)} =  -y^{(t)}_q,
 \end{align*}
     we have the following recursions:
    \begin{align}
        &{y^{(t)}}^\top =  {w^*}^\top\prod_{i=0}^{t-1}(I - \Sigma \Aa[i])X + \sum_{i=0}^{t-1}{\ddd^{(i)}}^\top \Sigma \Aa[i] 
\prod_{j=i+1}^{t-1} (I - \Sigma \Aa[j])X,~\label{eq:baseformulaone}\\
        &{y^{(t)}_{q}}^\top = {y_q}^\top - {w^*}^\top\prod_{i=0}^{t-1}(I - \Sigma \Aa[i])x_q - \sum_{i=0}^{t-1}{\ddd^{(i)}}^\top \Sigma \Aa[i] 
\prod_{j=i+1}^{t-1} (I - \Sigma \Aa[j]) x_q,\label{eq:baseformulatwo}
    \end{align}
    with the convention that $\prod_{0}^{-1} = 1$ and $\sum_0^{-1} = 0$.
\end{theorem}
\begin{proof}
    We show this by induction on $t$. For $t=0$, note that ${w^*}^\top X = y^{(0)}$ and $y^{(0)}_q = 0 = y_q - {w^*}^\top x_q$. For the step of induction, suppose we have Equations~\eqref{eq:baseformulaone} and~\eqref{eq:baseformulatwo} for $t-1$. Then, from the update rule
    \begin{align*}
        \Zz[t] = \Zz[t-1] - \frac{1}{n}\Pl[t] M\Zz[t] \Aa[t] {\Zz[t]}^\top, 
    \end{align*}
    we get
    \begin{align*}
        {y^{(t+1)}}^\top &= {y^{(t)}}^\top - \frac{1}{n}{y^{(t)}}^\top X^\top \Aa[t] X\\
        & = {w^*}^\top\prod_{i=0}^{t-1}(I - \Sigma \Aa[i])X - {w^*}^\top\prod_{i=0}^{t-1}(I - \Sigma \Aa[i])(XX^\top)\Aa[i]X\\
        &+ \sum_{i=0}^{t-1}{\ddd^{(i)}}^\top \Sigma \Aa[i] 
\prod_{j=i+1}^{t-1} (I - \Sigma \Aa[j])X - 
\sum_{i=0}^{t-1}{\ddd^{(i)}}^\top \Sigma \Aa[i] 
\prod_{j=i+1}^{t-1} (I - \Sigma \Aa[j])(X X^\top)\Aa[t] X\\
        & = {w^*}^\top\prod_{i=0}^{t}(I - \Sigma \Aa[t])X + \sum_{i=0}^{t-1}{\ddd^{(i)}}^\top \Sigma \Aa[i] 
\prod_{j=i+1}^{t-1} (I - \Sigma \Aa[j])X,
    \end{align*}
    where we used the fact that $XX^\top = \Sigma$. Moreover
    \begin{align*}
        y^{(t+1)}_q &= y^{(t)}_q - \frac{1}{n}{y^{(t)}}^\top X^\top Q x_q\\
        &= y_q - {w^*}^\top\prod_{i=0}^{t-1}(I - \Sigma \Aa[i])x_q - \sum_{i=0}^{t-1}{\ddd^{(i)}}^\top \Sigma \Aa[i] 
\prod_{j=i+1}^{t-1} (I - \Sigma \Aa[j]) x_q
    \\
    &- {w^*}^\top\prod_{i=0}^{t-1}(I - \Sigma \Aa[i])(XX^\top)\Aa[i]x_q + \sum_{i=0}^{t-1}{\ddd^{(i)}}^\top \Sigma \Aa[i] 
\prod_{j=i+1}^{t-1} (I - \Sigma \Aa[j])(XX^\top)\Aa[j]x_q\\
&= y_q - {w^*}^\top\prod_{i=0}^{t}(I - \Sigma \Aa[i])x_q - \sum_{i=0}^{t-1}{\ddd^{(i)}}^\top \Sigma \Aa[i] 
\prod_{j=i+1}^{t} (I - \Sigma \Aa[j]) x_q,
    \end{align*}
\end{proof}
which completes the step of induction.

\begin{lemma}[moments of the Gaussian covariance]\label{lem:basemomentestimation}
   Given $n\geq 4k^2 d^2$, for iid normal random vectors $x_1, \dots, x_n \sim \mathcal N(0,I)$ and data covariance $\Sigma = \frac{1}{n}\sum_{i=1}^n x_i {x_i}^\top$ we have
    \begin{align*}
        \E{}{\frac{1}{n}\sum_{i=1}^n x_i {x_i}^\top} = \alpha_{n,k,d} I,
    \end{align*}
    for 
    \begin{align*}
        1\leq \alpha_{n,d,k} \leq 1 + \frac{4kd}{\sqrt n}.
    \end{align*}
\end{lemma}
\begin{proof}
    Note that for $2k$ (correlated) normal variables $u_1,\dots,u_{2k}$, from Isserlis' theorem we have
    \begin{align*}
        \E{}{u_1\dots u_{2k}} = \sum_{p \in \mathcal P(2k)} \E{}{u_{p^1_1}u_{p^1_2}}\E{}{u_{p^2_1}u_{p^2_2}}\dots\E{}{u_{p^k_1}u_{p^k_2}},
    \end{align*}
    where $\mathcal P(2k)$ is the set of allocations of $\{1,2,\dots,2k\}$ into unordered pairs $\Big((p^1_1,p^1_2), p^2 = (p^2_1,p^2_2),\dots,(p^k_1,p^k_2)\Big)$. For an array of random matrices $\Big(M^{(1)}, \dots, M^{(2k)}\Big)$ and allocation $p \in \mathcal P(2k)$, let $M(p) = (M(p)^{(1)}, \dots, M(p)^{(2k)})$ be the set of random matrices where for each pair $(p^i_1, p^i_2) \in p$, $M(p)_{p^i_1}$ and $M(p)_{p^i_2}$ have the same joint distribution as $M_{p^i_1}$ and $M_{p^i_2}$, while for $i\neq j$, $(M(p)_{p^i_1}, M(p)_{p^i_2})$ and $(M(p)_{p^j_1}, M(p)_{p^j_2})$ are indepdenent from each other. Equipped with this notation, we now apply Isserlis' theorem to each summand of the product of $2k$ matrices $M_1,\dots,M_{2k}$, which provides us with a similar expansion of the expectation of the product of matrices:
    \begin{align}
        \E{}{M^{(1)}\dots M^{(2k)}} 
        =
        \sum_{p \in \mathcal P(2k)} \E{}{M(p)^{(1)}M(p)^{(2)}\dots M(p)^{(2k)}}.\label{eq:matrixexp}
    \end{align}
    Now using Equation~\eqref{eq:matrixexp}
    \begin{align}
        \E{}{(\sum_{i=1}^n x_i x_i^\top)^{k}} = 
        \sum_{p \in \mathcal P(k)}\sum_{(\multiiindex_1,\dots,\multiiindex_{k})\in [n]^{k}} \E{}{x^{(1)}_{\multiiindex_1}(p)x^{(1)}_{\multiiindex_1}(p)^\top x^{(2)}_{\multiiindex_3}(p)x^{(2)}_{\multiiindex_4}(p)^\top\dots x^{(k)}_{\multiiindex_{k}}(p)x^{(k)}_{\multiiindex_{k}}(p)^\top}\label{eq:firstsum}
    \end{align}
    where $x^{(i)}(p)$ for an allocation $p\in \mathcal P(2k)$ is defined similarly to $M^{(i)}(p)$ above. Now consider the graph $\mathcal G_p$ with vertices $\{1,\dots,k\}$ where we put an edge between $j$ and $k$ if one of the indices $(\multiiindex_{2j-1}, \multiiindex_{2j})$ is paired with $(\multiiindex_{2k-1}, \multiiindex_{2k})$ according to $p$.
    A key idea that we use here is considering the cycle structure of $\mathcal G_p$. It is clear that each vertex has degree exactly two, hence it is decomposed into a number of cycles. 

    note that for the multi-indices $(\multiiindex_1, \dots, \multiiindex_{2k})$ in the sum~\eqref{eq:firstsum} and a pair $(p^j_1, p^j_2)$, if $\multiiindex_{p^j_1}$ and $\multiiindex_{p^j_2}$ are different, then the corresponding $x_{p^j_1}$ and $x_{p^j_2}$ are independent. Hence, the expectation is zero: 
    \begin{align*}
        \E{}{x^{(1  )}_{\multiiindex_1}(p)x^{(1)}_{\multiiindex_1}(p)^\top x^{(2)}_{\multiiindex_2}(p)x^{(2)}_{\multiiindex_2}(p)^\top\dots x^{(k)}_{\multiiindex_{k}}(p)x^{(k)}_{\multiiindex_{k}}(p)^\top} = 0.
    \end{align*}
    therefore, for each pair $p^i = (p^i_1, p^i_2)$, $\multiiindex_{p^i_1}=\multiiindex_{p^i_2}$. 
    This means that if there is an edge between $j_1$ and $j_2$ in $\mathcal G_p$, then $\multiiindex_{j_1} = \multiiindex_{j_2}$.
    Therefore, for a cycle $C = (j_1, \dots, j_r)$ in $\mathcal G_p$ we have
    $\multiiindex_{j_1} = \multiiindex_{j_2} = \dots = \multiiindex_{j_r}$. Note that we have exactly $n$ choices for the value of these indices. 

    Hence, given a pairing $p$, the number of different ways of picking the multi-index $\multiiindex$ such that the expectation of its corresponding term is not zero is exactly 
    \begin{align*}
        n^{C(\mathcal G_p)},
    \end{align*}
    where $C(\mathcal G_p)$ is the number of cycles in $\mathcal G$. On the other hand, all of the non-zero terms have equal expectation. 
    Therefore, we can write the expectation in Equation~\eqref{eq:firstsum} as
    \begin{align}
        \E{}{(\sum_{i=1}^n x_i x_i^\top)^{k}} = 
        \sum_{p \in \mathcal P(k)} n^{C(\mathcal G_p)} \E{}{x^{(1)}_{1}(p)x^{(1)}_{1}(p)^\top x^{(2)}_{1}(p)x^{(2)}_{1}(p)^\top\dots x^{(k)}_{1}(p)x^{(k)}_{1}(p)^\top}.\label{eq:firstexpansion}
    \end{align}

   Next, we make the observation that for taking expectation with respect to a pair, we can substitute both of the vectors in that pair by one of vectors in the standard basis, and sum the results. More rigorously, for a fixed matrix $A$, vector $v$, and scalar $\alpha$ we have
   \begin{align*}
       &\E{}{x_1^\top A x_1} = \tr{A} = \sum_i e_i^\top A e_i,\\
       &\E{}{x_1 v^\top x_1} = v = \sum_i e_i v^\top e_i = v,\\
       &\E{}{x_1^\top v x_1^\top} = v^\top = \sum_i e_i^\top v e_i^\top,\\
       &\E{}{x_1 \alpha x_1^\top} = \alpha I = \sum_i e_i \alpha e_i^\top.\numberthis\label{eq:unrolleq}
   \end{align*}
   We can use this observation to unroll the expectation in Equation~\eqref{eq:matrixexp} as a sum. For example if $j_1, j_2$ are paired according to $p$, then
   \begin{align*}
    &\E{}{x^{(1)}_{1}(p)x^{(1)}_{1}(p)^\top x^{(2)}_{1}(p)x^{(2)}_{1}(p)^\top\dots x^{(k)}_{1}(p)x^{(k)}_{1}(p)^\top}
    \\
    &= 
        \E{}{x^{(1)}_{1}(p)\dots x^{(j_1)}_{1}(p) \Big(x^{(j_1)}_{1}(p)^\top x^{(j_1 + 1)}_{1}(p) \dots x^{(j_2 - 1)}_{1}(p)^\top\Big) x^{(j_2)}_{1}(p) \dots x^{(k)}_{1}(p)^\top}.\\
    &= 
    \sum_{i=1}^n\E{}{x^{(1)}_{1}(p)\dots e_i \Big(x^{(j_1)}_{1}(p)^\top x^{(j_1 + 2)}_{1}(p) \dots x^{(j_2 - 1)}_{1}(p)^\top\Big) e_i \dots x^{(k)}_{1}(p)^\top}.    
   \end{align*}
    Unrolling the expectation using Equations~\eqref{eq:unrolleq}, we get
    \begin{align}
        \E{}{x^{(1)}_{1}(p)x^{(1)}_{1}(p)^\top x^{(2)}_{1}(p)x^{(2)}_{1}(p)^\top\dots x^{(k)}_{1}(p)x^{(k)}_{1}(p)^\top}
        = \sum_{(\multiiindex_1,\dots,\multiiindex_{2k})\in [n]^{2k},\ \forall (j_1, j_2)\in p, \multiiindex_{j_1} = \multiiindex_{j_2}} e_{\multiiindex_1} e_{\multiiindex_2}^\top e_{\multiiindex_3} e_{\multiiindex_4}^\top \dots e_{\multiiindex_{2k-1}} e_{\multiiindex_{2k}}^\top.\label{eq:secondexpansion}  
    \end{align}
    The first observation above is that for consecutive elements $e_{\multiiindex_j}^\top e_{\multiiindex_{j+1}}$ we should have $\multiiindex_j = \multiiindex_{j+1}$ otherwise the product is zero. Hence, the sum above is really on multiindices of size $k$.
    Based on this observation, we consider a graph $\mathcal G'_p$ corresponding to the allocation $p$ whose nodes are the pairs $(2,3), (4,5), \dots, (2k-2, 2k-1), (1,2k)$, and we connect two nodes $(j_1, j_2)$ and $(i_1, i_2)$ in $\mathcal G'_p$ if either $j_1$ or $j_2$ is paired with $i_1$ or $i_2$ according to $p$. Then, similar to our argument for $\mathcal G$, for a cycle $(j_1, j_2, \dots, j_r)$ in $\mathcal G'_p$ in order for the term in Equation~\eqref{eq:secondexpansion} to be non-zero, we should have $\multiiindex_{j_1} = \multiiindex_{j_2} = \dots= \multiiindex_{j_r}$. Therefore, the total number of choices for the multi-index $\multiiindex$ in Equation~\eqref{eq:secondexpansion} is $n^{C(\mathcal G'_p)}$, in which case the term $e_{\multiiindex_1} e_{\multiiindex_2}^\top e_{\multiiindex_3} e_{\multiiindex_4}^\top \dots e_{\multiiindex_{2k-1}} e_{\multiiindex_{2k}}^\top$ is equal to $I$. Therefore
    \begin{align}
        \E{}{x^{(1)}_{1}(p)x^{(1)}_{1}(p)^\top x^{(2)}_{1}(p)x^{(2)}_{1}(p)^\top\dots x^{(k)}_{1}(p)x^{(k)}_{1}(p)^\top} = n^{C(\mathcal G'_p)} I.\label{eq:secondexpansiontwo}
    \end{align}
    Combining Equations~\eqref{eq:secondexpansiontwo} and~\eqref{eq:firstexpansion}:
    \begin{align*}
        \E{}{(\sum_{i=1}^n x_i x_i^\top)^{k}} = 
        \sum_{p \in \mathcal P(k)} n^{C(\mathcal G_p)} d^{C(\mathcal G'_p)-1} I.
    \end{align*}
    Now we need to estimate the number of cycles in the two graphs $\mathcal G_p$ and $\mathcal G'_p$ and how they interact with the choice of the allocation $p$. Note that each pair in the allocation $p$ translates into an edge in $\mathcal G_p$ and $\mathcal G'_p$ and can be a self-loop (from a node to itself). Another point is that from the definition of $\mathcal G_p$ and $\mathcal G'_p$, every node has degree exactly two. The key idea that we use here is that the total number of loops in the two graphs is bounded by $k+1$. The reason is that each pair in $p$ can be a self-loop in at most one of the graphs (this is true from the definition of the graphs), and each self-loop in one of the graph reduces the number of possible cycles in the other graph; suppose the number of self-loops in $\mathcal G_p$ is $r$. Then the rest of the $k-r$ nodes can at most divide into cycles of length two. This means $C(\mathcal G_p) \leq r + \lfloor \frac{k-r}{2}\rfloor$. On the other hand, note that each self-loop in $\mathcal G_p$ is created by a pair $(2i-1,2i)$ in $p$, which is an edge between two consecutive nodes $((2i-2,2i-1), (2i,2i+1))$ in $\mathcal G'_p$. Such an edge reduces the number of connected components of $G'_p$ by one. But as $G'_p$ mentioned, $G'_p$ decomposes into a number of loops, hence the number of its connected components is equal to the number of its loops. Therefore, the number of loops in $G'_p$ is at most $k-r$, i.e. $C(\mathcal G'_p) \leq k-r + 1$. 
    
    Next, we upper bound the number of $p$'s for which $\mathcal G_p$ has at least $r$ self loops: We have at most $\binom{k}{r}$ number of choices for the self-loop nodes. Then, we are left with $k - r$ nodes, including $2(k-r)$ pairs of indices (according to the definition of $\mathcal G_p$). This means there are at most $C_{k-r}$ choices for the rest of the graph, where $C_{k-r}$ is the number of different ways that we can allocate $\{1,2,\dots, 2(k-r)\}$ into $(k-r)$ pairs. It is easy to see that $C_{k-r} = (2k-2r)!!$. Therefore, there are at most $\binom{k}{r} (2k-2r)!!$ choices of $p$ which results in $\mathcal G_p$ with at least $r$ self-loops. Putting everything together
    \begin{align*}
        n^k\alpha_{n,d,k} &\leq \sum_{p \in \mathcal P(k)} n^{r + \lfloor \frac{k-r}{2}\rfloor} d^{k-r}\\
        & \leq \sum_{r=0}^k \binom{k}{r} (2k-2r)!! n^{\frac{k+r}{2}} d^{s}\\
        & = \sum_{s=0}^k \binom{k}{s} (2s)!! n^{k - \frac{s}{2}}d^{s}\\
        & \leq \sum_{s=0}^k \frac{k^s}{\sqrt{2\pi s}(s/e)^s} \sqrt{2s}(\frac{2s}{e})^s n^{k} (\frac{d}{\sqrt n})^{s}\\
        &\leq \sum_{s=0}^k (2k)^s n^k (\frac{d}{\sqrt n})^s.\numberthis\label{eq:derivationone}
    \end{align*}
    Now from the assumption $n\geq 4k^2 d^2$, we get
    \begin{align*}
        n^k \alpha_{n,d,k} 
        \leq n^k (1 + \frac{4kd}{\sqrt n}).
    \end{align*}
    The proof of the upper bound on $\alpha_{n,d,k}$ is complete. The lower bound simply follows from Jensen inequality. 
    
\end{proof}

\begin{lemma}
    The loss can be written as
    \begin{align*}
        \loss = \E{X}{\Big\|\Sigma^{-1/2}\prod_{i=0}^{L-1}(I - \Sigma \Aa[i])\Sigma^{1/2}\Big\|^2} + 
    \E{X}{\Big\|\sum_{i=0}^{L-1}d_i^\top \Sigma \Aa[i] 
\prod_{j=i+1}^{L-1} (I - \Sigma \Aa[j])\Sigma^{1/2}\Big\|^2}.
    \end{align*}
\end{lemma}
\begin{proof}
    Note that 
    \begin{align*}
        \E{}{(y_q - y^{(L)}_q)^2}
        &= \E{w^*, X, x_q}{\Big({w^*}^\top\prod_{i=0}^{L-1}(I - \Sigma \Aa[i])x_q + \sum_{i=0}^{L-1}d_i^\top \Sigma \Aa[i] 
\prod_{j=i+1}^{L-1} (I - \Sigma \Aa[j]) x_q\Big)^2}.
    \end{align*}
    Now note that $w^*$ is independent of $\sum_{i=0}^{L-1}d_i^\top \Sigma \Aa[i] 
\prod_{j=i+1}^{L-1} (I - \Sigma \Aa[j]) x_q$. Therefore, taking expectation with respect to $w^*$:
\begin{align*}
    \E{}{(y_q - y^{(L)}_q)^2} 
    &=  \E{w^*, x_q, X}{\Big({w^*}^\top\prod_{i=0}^{L-1}(I - \Sigma \Aa[i])x_q\Big)^2} + 
    \E{x_q, X}{\Big(\sum_{i=0}^{t-1}d _i^\top \Sigma \Aa[i] 
\prod_{j=i+1}^{L-1} (I - \Sigma \Aa[j]) x_q\Big)^2}\\
    &=  \E{x_q, X}{x_q^\top \Big(\prod_{i=0}^{L-1}(I - \Sigma \Aa[i])\Big)^2 x_q} + 
    \E{x_q, X}{\Big(\sum_{i=0}^{L-1}d _i^\top \Sigma \Aa[i] 
\prod_{j=i+1}^{L-1} (I - \Sigma \Aa[j]) x_q\Big)^2}.
\end{align*}
Finally taking expectation with respect to $x_q$:
\begin{align*}
    \E{}{(y_q - y^{(L)}_q)^2}
    = \E{X}{\Big\|\Sigma^{-1/2}\prod_{i=0}^{L-1}(I - \Sigma \Aa[i])\Sigma^{1/2}\Big\|^2} + 
    \E{X}{\Big\|\sum_{i=0}^{L-1}d _i^\top \Sigma \Aa[i] 
\prod_{j=i+1}^{L-1} (I - \Sigma \Aa[j])\Sigma^{1/2}\Big\|^2}
\end{align*}

\end{proof}

\begin{corollary}\label{cor:lossform}
    The loss for loop Transformer is
    \begin{align*}
        L(\Aaa, \ddd) = \E{X}{\tr{(I - A^{1/2}\Sigma A^{1/2})^{2L}}} + 
    \E{X}{\Big\|\sum_{i=0}^{L-1}\ddd^\top \Sigma \Aaa 
 (I - \Sigma \Aaa)^{L-1-i}\Sigma^{1/2}\Big\|^2}.
    \end{align*}
\end{corollary}
\begin{proof}
    Just note that 
    \begin{align*}
        L(\Aaa, \ddd) &= \E{X}{\tr{(I - \Sigma^{1/2} \Aaa\Sigma^{1/2})^{2L}}} + 
    \E{X}{\Big\|\sum_{i=0}^{L-1}\ddd^\top \Sigma \Aaa 
 (I - \Sigma \Aaa)^{L-1-i}\Sigma^{1/2}\Big\|^2}\\
 &= \E{X}{\tr{(I - A^{1/2}\Sigma A^{1/2})^{2L}}} + 
    \E{X}{\Big\|\sum_{i=0}^{L-1}\ddd^\top \Sigma \Aaa 
 (I - \Sigma \Aaa)^{L-1-i}\Sigma^{1/2}\Big\|^2}.
    \end{align*}
\end{proof}

\begin{lemma}[Restatement of Lemma~\ref{lem:momentcontrol}]
    
 Suppose $\forall i\in [n], \tilde x_i\sim \mathcal N(0,\Sigma^*)$. Consider the eigen-decomposition  $\Sigma^* = \sum_{i=1}^d \lambda_i u_i u_i^\top$ with eigenvalues $\lambda_1 \geq \lambda_2\geq \dots\geq \lambda_d$. Then, $\E{}{(\frac{1}{n}\sum_{i=1}^n \tilde x_i \tilde x_i^\top)^{k}}$ can be written as 
    \begin{align*}
        \E{}{(\frac{1}{n}\sum_{i=1}^n \tilde x_i \tilde x_i^\top)^{k}} = \sum_{j=1}^d \alpha^{(j)}_{n,d,k} u_j u_j^\top,
    \end{align*}
    where for all $1\leq j\leq d$:
    \begin{align*}
        \lambda_j^k-\delta\lambda_1^k \leq \alpha^{(j)}_{n,d,k} \leq \lambda_j^k + \delta\lambda_1^k. 
    \end{align*}
\end{lemma}

\begin{proof}
In the proof of Lemma~\ref{lem:basemomentestimation}, we used Equations~\eqref{eq:unrolleq} to simplify the loss and write it as a sum over the normal basis vectors $e_i$. It is easy to see that we have the equivalence of Equations~\eqref{eq:unrolleq} for $\tilde x_i\sim \mathcal N(0,\Sigma^*)$ when $e_i$'s are replaced by $v_i = \sqrt{\lambda_i} u_i$.
    \begin{align*}
        \E{}{(\sum_{i=1}^n x_i x_i^\top)^{k}} = 
        \sum_{p \in \mathcal P(k)} n^{C(\mathcal G_p)} \sum_{\multiiindex\in [d]^{C(\mathcal G'_p)}} \Big(\prod_{c \in C(\mathcal G'_p), (1,2k)\notin c}\lambda_{\multiiindex_c}^{|c|}\Big) \lambda_{\multiiindex_{c^*}}^{|c^*|} v_{\multiiindex_{c^*}}v_{\multiiindex_{c^*}}^\top,
    \end{align*}
    where $c^*$ is the loop in $\mathcal G'_p$ that includes the node consisting of the first and the last indices , i.e. $(1,2k)$. But pushing the second sum to the product:
    \begin{align*}
        \E{}{(\sum_{i=1}^n x_i x_i^\top)^{k}} &= 
        \sum_{p \in \mathcal P(k)} n^{C(\mathcal G_p)} \Big(\prod_{c \in C(\mathcal G'_p), (1,2k)\notin c} \Big(\sum_{i=1}^d\lambda_{i}^{|c|} \Big) \Big)\Big(\sum_{i=1}^d\lambda_{i}^{|c^*|-1}v_{\multiiindex_{c^*}}v_{\multiiindex_{c^*}}^\top\Big)\\
        &=\sum_{p \in \mathcal P(k)} n^{C(\mathcal G_p)} \Big(\prod_{c \in C(\mathcal G'_p), (1,2k)\notin c} \Big(\sum_{i=1}^d\lambda_{i}^{|c|} \Big) \Big)\Big({\Sigma^*}^{|c^*|}\Big).
    \end{align*}
    Now if we upper bound all the eigenvalues $\lambda_i$ in the sum $\Big(\sum_{i=1}^d\lambda_{i}^{|c|} \Big)$ above, we have similar to Equation~\eqref{eq:derivationone} in the proof of Lemma~\ref{lem:basemomentestimation}.
    \begin{align*}
        n^k \alpha^{(j)}_{n,d,k} &= u_j^\top \E{}{(\sum_{i=1}^n x_i x_i^\top)^{k}} u_j\\
        &\leq \sum_{p \in \mathcal P(k)} n^{C(\mathcal G_p)} \Big(\prod_{c \in C(\mathcal G'_p), (1,2k)\notin c} d\lambda_{1}^{|c|} \Big)\lambda_j^{|c^*|}\\
        &\leq \sum_{p \in \mathcal P(k)} n^{C(\mathcal G_p)}d^{C(\mathcal G'_p)-1} \lambda_j^k\\
        &\leq n^k \lambda_j^k + n^k \lambda_1^k\frac{4kd}{\sqrt n},
    \end{align*}
    and similarly
    \begin{align*}
        n^k \alpha^{(j)}_{n,d,k} \geq n^k \lambda_j^k (1 - \frac{4kd}{\sqrt n}).
    \end{align*}
\end{proof}

\begin{lemma}[Changing the covariance matrix]\label{lem:momentcontrolhelper}
    We can write the first part of the loss $\E{X}{\tr{(I - \Aaa^{1/2} \Sigma \Aaa^{1/2})^{2L}\Sigma^{1/2}}}$ as
    \begin{align*}
         \E{X}{\tr{(I -  \Aaa^{1/2}\Sigma A^{1/2})^{2L}}} = \tr{\E{X}{(I - \tilde \Sigma)^{2L}}},
    \end{align*}
    for 
    \begin{align*}
        \tilde \Sigma = \frac{1}{n}\sum_{i=1}^n \tilde x_i \tilde x_i^\top,
    \end{align*}
    where $\forall i, \tilde x_i \sim \mathcal N(0,A^{1/2}\covariance A^{1/2})$.
\end{lemma}
\begin{proof}
    
    Defining $\tilde x_i = A^{1/2}x_i$, then $\tilde x_i\sim \mathcal N(0,A^{1/2}\Sigma^*A^{1/2})$, and
    \begin{align*}
        \E{X}{(I - \Aaa^{1/2}\Sigma \Aaa^{1/2})^{2L}} = \E{X}{(I - \tilde \Sigma)^{2L}},
    \end{align*}
    for
    \begin{align*}
        \tilde \Sigma = \frac{1}{n}\sum_{i=1}^n \tilde x_i \tilde x_i^\top.
    \end{align*}
    This finishes the proof.
\end{proof}

\begin{lemma}[Restatement of Lemma~\ref{lem:eigapprox}]
For all $i\in [d]$, $\E{}{ A^{-1/2}(A^{1/2}\Sigma A^{1/2} - I)^k A^{1/2}}$ can be written as  

$$ \E{}{ (A^{1/2}\Sigma A^{1/2} - I)^k } = \sum_{i=1}^d \beta_i^{(k)} u_i u_i^\top.$$ 
Furthermore, for 
\begin{align*}
    \delta^k = \frac{4kd}{\sqrt n},
\end{align*}
we have
\begin{align*}
    (\lambda_i - 1)^k - \delta^k (\lambda_1 + 1)^k \leq \beta^{(k)}_i \leq (\lambda_i - 1)^k + \delta^k (\lambda_1 + 1)^k,
\end{align*}
where $\lambda_i$ is the $i$th eigenvalue of $A^{1/2}\Sigma^* A^{1/2}$, where recall $\Sigma^*$ is the covariance matrix of $x_i$'s.
\end{lemma}
\begin{proof}
    The idea is to open up the power of matrix and estimate each of the terms separately:
    \begin{align}
        \E{}{ A^{-1/2}(A^{1/2}\Sigma A^{1/2} - I)^k A^{1/2}} = \sum_{i=1}^k (-1)^{i} \binom{k}{i}\E{}{ A^{-1/2}(A^{1/2}\Sigma A^{1/2})^i A^{1/2}}.\label{eq:expansioneq}
    \end{align}
    The proof Directly follows from Lemmas~\ref{lem:momentcontrol} and~\ref{lem:momentcontrolhelper}.
\end{proof}

\begin{theorem}[Optimal solution]
     Suppose $\{\Aopt, \dopt\}$ are a global minimizer for $L(A,d)$. Then, under condition $\delta d^{1/(2L)} < \frac{1}{2}$,  
     \begin{enumerate}
         \item 
         \begin{align*}
             L(\Aopt, \dopt) \leq d(2\delta)^{2L}.
         \end{align*}
         \item
         \begin{align}
             \|{\Aopt}^{1/2}\covariance {\Aopt}^{1/2}  - I\| \leq 4\delta d^{1/(2L)}, d^\text{opt} = 0.\label{eq:closeness}
         \end{align}
         \item 
         \begin{align*}
                &(1 - 8\delta d^{1/(2L)}){\covariance}^{-1} \preccurlyeq \Aopt \preccurlyeq (1+8\delta d^{1/(2L)}){\covariance}^{-1}.\numberthis\label{eq:keycond}    
        \end{align*}
     \end{enumerate}
\end{theorem}
\begin{proof}
     Recall the form of the loss from Corollary~\ref{cor:lossform}. First, note that the second term, $\E{X}{\Big\|\sum_{i=0}^{L-1}\ddd^\top \Sigma \Aaa 
 (I - \Sigma \Aaa)^{L-1-i}\Big\|^2}$, is always positive when $\ddd \neq 0$ and is zero if $\ddd = 0$. Therefore $\ddd^{\text{opt}} = 0$. Next, we show the upper bound on the optimal loss:
 \begin{align*}
     L(\Aopt, \dopt) 
     = L(\Aopt, 0) \leq L({\covariance}^{-1},0).
 \end{align*}
 Recall $\beta^{(2L)}_i$ is the $i$th eigenvalue of $\E{}{ A^{-1/2}(A^{1/2}\Sigma A^{1/2} - I)^k A^{1/2}}$.
 Note that from Lemma~\ref{lem:eigapprox}
 \begin{align*}
     -(2\delta)^{2L} \leq \beta^{(2L)}_i \leq (2\delta)^{2L}.
 \end{align*}
 On the other hand, $L(\Aaa, \ddd) = \sum_{i=1}^d \beta^{(2L)}_i$, which means $L(I,0) \leq d(2\delta)^{2L}$.
To show Equation~\eqref{eq:closeness}, suppose $|\lambda_1 - 1| \geq 4\delta d^{1/(2L)}$. Then, using Lemma~\ref{lem:eigapprox}
\begin{align*}
   \beta^{(k)}_i \geq (\lambda_i - 1)^k - \delta^k (\lambda_1 + 1)^k \geq (4\delta d^{1/(2L)})^{2L} - \delta^{2L} (1+4\delta d^{1/(2L)})^{2L} > d(2\delta)^{2L}.
\end{align*}
where we used the inequality $\delta d^{1/(2L)} < \frac{1}{2}$. Finally this implies
\begin{align}
             \|{\Aopt}^{1/2}\covariance {\Aopt}^{1/2}  - I\| \leq 4\delta d^{1/(2L)}, d^\text{opt} = 0,
         \end{align}
which means
\begin{align*}
                &(1 - 8\delta d^{1/(2L)}){\covariance}^{-1} \preccurlyeq \Aopt \preccurlyeq (1+8\delta d^{1/(2L)}){\covariance}^{-1}.
\end{align*}
\end{proof}

\begin{lemma}[Small loss implies close to optimal]\label{lem:proximity}
    For $\epsilon > 2\delta$ and parameter $A$ suppose we have $\\L(\Aaa, 0) \leq \epsilon^{2L} - \delta^{2L} (\epsilon + 2)^{2L}$. Then
     \begin{align*}
          (1 - (4\epsilon + 16\delta d^{1/(2L)}))A^{opt} \preccurlyeq A \preccurlyeq (1 + (4\epsilon + 16\delta d^{1/(2L)}))A^{opt}.
    \end{align*}
\end{lemma}
\begin{proof}
    First note that we should have $|\lambda_1 - 1| \leq \epsilon$. This is because from Lemma~\ref{lem:eigapprox} we get
    \begin{align*}
        \epsilon^{2L} - \delta^{2L} (\epsilon + 2)^{2L} \geq L(A,0) \geq \beta_1^{(2L)} \geq (\lambda_1 - 1)^{2L} - \delta^{2L} (\lambda_1 + 1)^{2L},
    \end{align*}
    which implies 
    \begin{align*}
        |\lambda_1 - 1|\leq \epsilon.
    \end{align*}
    Then again using Lemma~\ref{lem:eigapprox} this time for $\lambda_i$ we should also have
    \begin{align*}
        \epsilon^{2L} - \delta^{2L} (\epsilon + 2)^{2L} &\geq L(A,0) \geq \beta_i^{(2L)} \geq (\lambda_i - 1)^{2L} - \delta^{2L} (\lambda_1 + 1)^{2L} \\
        &\geq (\lambda_i - 1)^{2L} - \delta^{2L} (\epsilon + 2)^{2L},
    \end{align*}
    which implies
    $|\lambda_i - 1| \leq \epsilon$. Hence, overall we showed
    \begin{align*}
        \|A^{1/2} \covariance A^{1/2}  - I\| \leq \epsilon,
    \end{align*}
    
Expanding this term, we get terms of the form $\E{X}{A^{-1/2} (A^{1/2}\Sigma A^{1/2})^{\ell_1} A^{-1} (A^{1/2}\Sigma A^{1/2})^{\ell_2} A^{1/2}}$ for $\ell_1 \geq 1$:
\begin{align}
    &\E{X}{(I-\Sigma A)^{i} \Sigma (I-\Sigma A)^{2L-1 - i}}\\
    &= \sum_{0 \leq \ell_1 \leq i}\sum_{0 \leq \ell_2 \leq 2L-1-i} \binom{2L-1-i}{\ell_1} \binom{i}{\ell_2} (-1)^{\ell_1 + \ell_2} \E{X}{A^{-1/2} (A^{1/2}\Sigma A^{1/2})^{\ell_1 + 1} A^{-1} (A^{1/2}\Sigma A^{1/2})^{\ell_2} A^{1/2}}\label{eq:correq}
\end{align}
   
  Now similar to the proof of Lemma~\ref{lem:momentcontrol}, we calculate the expectation of each term of this form. The subtle point here is that even though there is the matrix $A^{-1}$ in between the random matrices, it has shared eigenvalues as the covariance matrix of the gaussians, hence the computation goes through similarly expect that for the cycle $\tilde c$ in $\mathcal G'_p$ which includes the vertex $(2\ell_1, 2\ell_1 + 1)$ generates a $\lambda_i^{|\tilde c| - 1}$ instead of $\lambda_i^{|\tilde c|}$. More rigorously, for $A = \sum \lambda_i u_i u_i^\top$ We can argue that
  \begin{align*}
      \E{X}{A^{-1/2} (A^{1/2}\Sigma A^{1/2})^{\ell_1} A^{-1} (A^{1/2}\Sigma A^{1/2})^{\ell_2} A^{1/2}} = \sum_i \gamma_j u_j u_j^\top,
  \end{align*}
    where for $k = \ell_1 + \ell_2 + 1$, similar to the proof of Lemma~\ref{lem:basemomentestimation}:
  \begin{align*}
        n^k \gamma_j &= u_j^\top \E{}{(\sum_{i=1}^n \tilde x_i \tilde x_i^\top)^{k}} u_j\\
        &\leq \sum_{p \in \mathcal P(k)} n^{C(\mathcal G_p)} \Big(\prod_{c \in C(\mathcal G'_p), (1,2k)\notin c, c \neq \tilde c} d\lambda_{1}^{|c| - \mathrm 1\{(2\ell_1, 2\ell_1 + 1) \in c\}} \Big)\lambda_j^{|c^*| - \mathrm 1\{(2\ell_1, 2\ell_1 + 1) \in c^*\}}\\
        &\leq \sum_{p \in \mathcal P(k)} n^{C(\mathcal G_p)} d^{C(\mathcal G'_p)-1} \Big(\prod_{c \in C(\mathcal G'_p), (1,2k)\notin c, c\not \tilde c} \lambda_{1}^{|c| - \mathrm 1\{(2\ell_1, 2\ell_1 + 1) \in c\}} \Big)\lambda_j^{|c^*| - \mathrm 1\{(2\ell_1, 2\ell_1 + 1) \in c^*\}}\\
        &\leq n^k \lambda_j^{k-1} + n^k \lambda_1^{k-1}\frac{4kd}{\sqrt n},
    \end{align*}
    and similarly
    \begin{align*}
        n^k \gamma_j \geq n^k \lambda_j^{k-1} - n^k \lambda_1^{k-1}\frac{4kd}{\sqrt n},
    \end{align*}
    which implies
    \begin{align*}
        \lambda_j^{k-1} - \delta^k{\lambda_1}^{k-1} \leq \gamma_j \leq \lambda_j^{k-1} + \delta^k{\lambda_1}^{k-1}.
    \end{align*}
    Plugging this into Equation~\eqref{eq:correq} we get
    \begin{align*}
        \E{X}{(I-\Sigma A)^{i} \Sigma (I-\Sigma A)^{2L-1 - i}}
        &= \sum_{0 \leq \ell_1 \leq i}\sum_{0 \leq \ell_2 \leq 2L-1-i} \sum_j  \binom{2L-1-i}{\ell_1} \binom{i}{\ell_2} ((-1)^{\ell_1 + \ell_2} \lambda_j^{\ell_1 + \ell_2} + \delta^{2L} \lambda_1^{\ell_1 + \ell_2}) u_ju_j^\top\\
        &\preccurlyeq \sum_j \big((\lambda_j - 1)^{2L - 1} + \delta^{2L}(\lambda_j + 1)^{2L-1}\big) u_j u_j^\top,
    \end{align*}
    and similarly
    \begin{align*}
         \E{X}{(I-\Sigma A)^{i} \Sigma (I-\Sigma A)^{2L-1 - i}} \succcurlyeq 
         \sum_j \big((\lambda_j - 1)^{2L - 1} - \delta^{2L}(\lambda_1 + 1)^{2L-1}\big) u_j u_j^\top.
    \end{align*}
    Combining this with Equation~\eqref{eq:derformula} concludes the result.
\end{proof}


\begin{lemma}[Gradient dominance]\label{lem:graddominanceone}
    Suppose $ 8\delta d^{3/(4L)} \leq \epsilon \leq 1$ and $4d\delta^{4L-2} \leq 1$. Then if $L(A, 0) \geq \epsilon^{2L} + d\delta^{2L}(\epsilon + 2)^{2L}$, we have 
    \begin{align*}
        \|\nabla L(\Aaa, 0)\|^2 \geq \frac{1}{4}\epsilon^{4L-2}.
    \end{align*}
\end{lemma}
\begin{proof}
    

    Using Lemma~\ref{lem:eigapprox}, we have $\nabla_{\Aaa}L(\Aaa, 0) = \sum_i \gamma_i u_i u_i^\top$ such that
    \begin{align*}
        \|\nabla_{\Aaa}L(\Aaa, 0)\|^2 &=  \sum_{i=1}^d \gamma_i^2\\
        &\geq  \sum_i \Big((\lambda_i - 1)^{2L-1} - \delta^{2L}(\lambda_1 + 1)^{2L-1}\Big)^2\\
        &\geq  \sum_i \Big(\frac{1}{2}(\lambda_i - 1)^{4L-2} - \delta^{4L}(\lambda_1 + 1)^{4L-2}\Big)\\
        &=  \sum_i \frac{1}{2}(\lambda_i - 1)^{4L-2} - d\delta^{4L}(\lambda_1 + 1)^{4L-2}\\
        &\geq  \Big(\frac{1}{2}\sum_i (\lambda_i - 1)^{4L-2}\Big) - d\delta^{4L}(\lambda_1 + 1)^{4L-2}\numberthis\label{eq:initialder}
    \end{align*}
    Now using Lemma~\ref{lem:pregraddominance}
    \begin{align*}
        d\max_i (\lambda_i - 1)^{2L} \geq \sum_{i=1}^{2L} (\lambda_i - 1)^{2L} \geq \epsilon^{2L},
    \end{align*}
    or 
    \begin{align*}
        \max_i |\lambda_i - 1| \geq \epsilon/d^{1/(2L)}.
    \end{align*}
    But this implies 
    \begin{align*}
        d\delta^{4L}(\lambda_1 + 1)^{4L-2} \leq \frac{1}{4}(\lambda_1 -1)^{4L -2},
    \end{align*}
    as the assumption on $\epsilon$ implies
    \begin{align*}
        (1 + \frac{2}{\lambda_1 - 1})^{4L-2}
        \leq (1 + \frac{2}{\max_i|\lambda_1 - 1|})^{4L-2}
        \leq (1 + \frac{2d^{1/(2L)}}{\epsilon})^{4L-2}
        \leq (\frac{3d^{1/(2L)}}{\epsilon})^{4L}
        \leq \frac{1}{4 d\delta^{4L}}.
    \end{align*}
    Plugging this into Equation~\eqref{eq:initialder}:
    \begin{align*}
        \|\nabla_{\Aaa}L(\Aaa, 0)\|^2 \geq \frac{1}{4}\sum_i (\lambda_i - 1)^{4L-2}. 
    \end{align*}
    
    Using Lemma~\ref{lem:pregraddominance} and Holder, we get 
    \begin{align*}
        \sum_{i=1}^d (\lambda_i - 1)^{4L-2} &\geq \frac{1}{d^{1/(2L) - 1/(4L-2)}} \Big(\sum_{i=1}^d (\lambda_i - 1)^{2L}\Big)^{(2L-1)/L} \\
        &= \frac{1}{d^{(L-1)/(2L-1)}} \Big(\sum_{i=1}^d (\lambda_i - 1)^{2L}\Big)^{(2L-1)/L}\\
        &\geq \frac{1}{d^{(L-1)/(2L-1)}} \Big(\epsilon^{2L}\Big)^{(2L-1)/L}\\
        &= \frac{1}{d^{(L-1)/(2L-1)}} \epsilon^{4L-2},
    \end{align*}
    which completes the proof.
\end{proof}

\begin{theorem}[Restatement of Theorem~\ref{thm:gradientdominance}]
    For any $A$ that $L(A,0) \geq 2(4\delta)^{2L}$, we have the following gradient dominance condition:
    \begin{align*}
        \|\nabla_A L(A,0)\|^2 \geq \frac{1}{16}L(A,0)^{(2L-1)/L}.
    \end{align*}
\end{theorem}
\begin{proof}
    For $\epsilon > 4\delta$ we have $\epsilon^{2L} \geq d\delta^{2L}(\epsilon + 2)^{2L}$. Therefore, according to Lemma~\ref{lem:graddominanceone}, for $\epsilon > 4\delta$ if we have $L(A,0) \geq 2\epsilon^{2L}$, then 
    \begin{align}
        \|\nabla L(A,0))\|^2 \geq \frac{1}{4}\epsilon^{4L - 2}.\label{eq:resultt}
    \end{align}
    Therefore, for any $A$ if we have $(L(A,0)/2)^{1/(2L)} \geq 4\delta$, then if we define $\epsilon = (L(A,0)/2)^{1/(2L)}$, we have $L(A,0) = 2\epsilon^{2L}$, which then implies (from Equation~\eqref{eq:resultt})
    \begin{align*}
        \|\nabla L(A,0))\|^2 \geq \frac{1}{4} (\frac{L(A,0)}{2})^{(2L-1)/L}.
    \end{align*}
    Therefore, we showed that if $L(A,0)\geq 2(4\delta)^{2L}$, then $\|\nabla L(A,0)\|^2 \geq \frac{1}{16}L(A,0)^{(2L-1)/L}$.
   
\end{proof}

\begin{theorem}[Convergence of the gradient flow]
    Consider the gradient flow with respect to the loss $L(A,0))$:
    \begin{align*}
        \frac{d}{dt}A(t) = -\nabla_A L(A(t),0). 
    \end{align*}
    Then, for any $\xi \geq 2(4\delta)^{2L}$, after time $t\geq \Big(\frac{1}{\xi}\Big)^{(L-1)/L} (\frac{16 L}{L-1})^{(L-1)/(2L-1)}$ we have $L(A(t)) \leq \xi$.
\end{theorem}
\begin{proof}
    Let $f(t) = L(A(t),0)$. Then from Theorem~\ref{thm:gradientdominance}, if $f(t) \geq 2(4\delta)^{2L}$, then we have
    \begin{align*}
        f'(t) &= \langle \frac{d}{dt}A(t), \nabla_A L(A,0)\rangle\\
        &-\langle \nabla_A L(A,0), \nabla_A L(A,0)\rangle\\
        &= -\|\nabla_A L(A,0)\|^2\\
        &\leq \frac{1}{16}f(t)^{(2L-1)/L}.
    \end{align*}
    Therefore, if we define the ODE
    \begin{align*}
        &g(0) = f(0),\\
        &\forall t \geq 0, g'(t) = \frac{1}{16}g(t)^{(2L-1)/L},\numberthis\label{eq:gode}
    \end{align*}
    then we have
    \begin{align*}
        f(t) \leq g(t), \forall t\geq 0.
    \end{align*}
    Solving the ODE~\eqref{eq:gode} we get
    \begin{align*}
        g(t) = (\frac{16L}{L-1})^{L/(2L-1)} \frac{1}{(t + c)^{L/(L-1)}},
    \end{align*}
    for 
    \begin{align*}
        c = \frac{(16L/(L-1))^{(L-1)/(2L-1)}}{f(0)^{(L-1)/L}}
        = \frac{(16L/(L-1))^{(L-1)/(2L-1)}}{{L(A_0, 0)}^{(L-1)/L}}.
    \end{align*}
    Therefore, we get the following upper bound on $f$:
    \begin{align*}
        f(t)\leq (\frac{16L}{L-1})^{L/(2L-1)} \frac{1}{(t + c)^{L/(L-1)}} \leq 
        (\frac{16L}{L-1})^{L/(2L-1)} \frac{1}{t^{L/(L-1)}}.
    \end{align*}
    Therefore, to guarantee $f(t) \leq \xi$ we pick $t \geq \Big(\frac{1}{\xi}\Big)^{(L-1)/L} (\frac{16 L}{L-1})^{(L-1)/(2L-1)}$.
\end{proof}

\begin{theorem}[Restatement of Theorem~\ref{thm:outofdist}]\label{thm:outofdist1}
    Let $A^{opt},\ddd^{opt}$ be the global minimizers of the poplulation loss for looped Transformer with depth $L$ when the in-context input $\{x_i\}_{i=1}^n$ are sampled from $\mathcal N(0,\covariance)$ and $w^*$ is sampled from $\mathcal N(0,{\Sigma^*}^{-1})$. Suppose we are given an arbitrary linear regression instance $\mathcal I^{out} = \Big\{x^{out}_i,y^{out}_i\Big\}_{i=1}^n, w^{out, *}$ with input matrix $X^{out} = [x^{out}_1, \dots, x^{out}_n]$, query vector $x^{out}_q$, and label $y^{out}_q = {w^{out, *}}^\top x^{out}_q$. Then, if for parameter $0 < \zeta < 1$, the input covariance matrix $\Sigma^{out} = X^{out}{X^{out}}^\top$ of the out of distribution instance satisfies
    \begin{align}
       \zeta \covariance \preccurlyeq \Sigma^{out} \preccurlyeq (2-\zeta)\covariance,\label{eq:comparetocovariance}
    \end{align}
    we have the following instance-dependent bound on the out of distribution loss: 
    \begin{align*}
        &(\TF{Z^{out}_0} - y^{out}_q)^2 \\
        & \leq (1+16\delta d^{1/(2L)})^2(1 + 16\delta d^{1/(2L)} - \zeta)^{2L}\\
        &\times \Big\|x^{out}_q\Big\|_{\covariance}^2 \Big\|w^{out, *}\Big\|_{{\covariance}^{-1}}^2.
    \end{align*}    
\end{theorem}
\begin{proof}
    Note that from Lemma, the global optimum $A^{opt},\ddd^{opt}$ satisfies $u^{opt} = 0$ and 
     \begin{align*}
        (1 - 16\delta d^{1/(2L)})  {\Aopt}^{-1}  \preccurlyeq    {\covariance} \preccurlyeq (1 + 16\delta d^{1/(2L)})  {\Aopt}^{-1}.
    \end{align*}
    But combining this with Equation~\eqref{eq:comparetocovariance}, we get
   \begin{align*}
       \zeta(1 - 16\delta d^{1/(2L)}) {\Aopt}^{-1} \preccurlyeq \Sigma^{out} \preccurlyeq (2-\zeta)(1 + 16\delta d^{1/(2L)}){\Aopt}^{-1},
   \end{align*}
   which implies
   \begin{align*}
       -(1 + 16\delta d^{1/(2L)} - \zeta)I & \leq \Big(I - {A^{opt}}^{1/2}\Sigma^{out}{A^{opt}}^{1/2}\Big) \\
       &\leq (1 + 16\delta d^{1/(2L)} - \zeta)I.\numberthis\label{eq:nicebound}
   \end{align*}

    Furthermore, plugging in the formula of $y_q^{L}$ for the out of distribution instance, we have
    \begin{align*}
        &{(\TF{Z^{out}_0} - y^{out}_q)^2}  \\
        &= \Big({w^{out, *}}^\top{A^{out}}^{-1/2}(I - {A^{opt}}^{1/2}\Sigma^{out}{A^{opt}}^{1/2})^{L}{A^{opt}}^{1/2}x^{out}_{q}\Big)^2\\
        &\leq \Big({w^{out, *}}^\top {A^{opt}}^{-1} w^{out, *}\Big) \Big\| I - {A^{opt}}^{1/2}\Sigma^{out}{A^{opt}}^{1/2}\Big\|^{2L}\\
        &\times \Big({x^{out}_q}^\top A^{opt} x^{out}_q\Big)\\
        &\leq (1+16\delta d^{1/(2L)})^2\Big({w^{out, *}}^\top {\covariance}^{-1} w^{out, *}\Big) \\
        &\times \Big\|I - {A^{opt}}^{1/2}\Sigma^{out}{A^{opt}}^{1/2}\Big\|^{2L}\Big({x^{out}_q}^\top \covariance x^{out}_q\Big).
    \end{align*}
    But note that from Equation~\eqref{eq:nicebound}
     \begin{align*}
        &{(\TF{Z^{out}_0} - y^{out}_q)^2} \\
        &\leq (1+16\delta d^{1/(2L)})^2(1 + 16\delta d^{1/(2L)} - \zeta)^{2L}\\
        &\times \Big\|x^{out}_q\Big\|_{\covariance}^2 \Big\|w^{out, *}\Big\|_{{\covariance}^{-1}}^2.
     \end{align*}
\end{proof}

\end{document}